\documentclass[runningheads]{llncs}
\usepackage[T1]{fontenc}
\usepackage{graphicx}
\usepackage{booktabs}
\usepackage{amsmath,amssymb}
\newtheorem{assumption}{Assumption}
\usepackage{amsfonts}
\usepackage{appendix}
\usepackage{xcolor}
\newcommand{\rev}[1]{#1}

\begin{document}

\title{Noise-Adaptive High-Probability Regret Bounds for Online Convex Optimization}

\titlerunning{Noise-Adaptive High-Probability Regret Bounds for OCO}

\author{Wentao Zhang\inst{1} \and Yutong Zhang\inst{2} \and Wentao Mo\inst{1}}

\authorrunning{W. Zhang et al.}

\institute{
Tsinghua Shenzhen International Graduate School, Tsinghua University, China\\
\email{\{zhang-wt24,mow10\}@mails.tsinghua.edu.cn}
\and
College of Mathematics, Sichuan University, China\\
\email{yutongzhang@stu.scu.edu.cn}
}

\maketitle              

\begin{abstract}
We study high-probability regret bounds for online convex optimization (OCO) with strongly convex losses and establish three results that resolve open questions at the intersection of noise adaptivity, feedback structure, and constraint satisfaction. For the full-information setting with sub-Gaussian stochastic gradients, we prove a noise-adaptive high-probability regret bound in which the martingale deviation term scales with the noise level $\sigma$ rather than the gradient bound $G$, yielding a multiplicative improvement of $G/\sigma$ over the classical Azuma-Hoeffding baseline. Our analysis introduces an exponential supermartingale argument that bypasses the bounded-difference requirement of Freedman's inequality, enabling direct treatment of unbounded sub-Gaussian noise without truncation artifacts. For bandit feedback, we prove a minimax lower bound: the high-probability regret scales linearly in $\log(1/\delta)$, in contrast to the $\sqrt{\log(1/\delta)}$ confidence cost under full information. This constitutes a formal separation in the confidence cost of strongly convex OCO across feedback models. Regarding constrained OCO with stochastic constraints satisfying a Slater condition, we provide simultaneous high-probability guarantees for both cumulative regret and long-run constraint violation, achieving $\mathcal{O}(\sqrt{T\log(m/\delta)})$ regret and $\mathcal{O}(\sqrt{T}/(\zeta\delta) + m\sqrt{T\log(m/\delta)})$ violation. Synthetic experiments corroborate all theoretical predictions.

\end{abstract}

\section{Introduction}
Online convex optimization (OCO) provides a principled framework for sequential decision-making under uncertainty, with applications spanning online learning, portfolio selection, resource allocation, and real-time control \cite{shalev2025online,hazan2016introduction}. In the standard OCO protocol, a learner iteratively selects decisions from a convex set, incurs losses chosen by an adversary, and aims to minimize regret, defined as the cumulative excess loss relative to the best fixed decision in hindsight. When the loss functions are $\alpha$-strongly convex, classical results establish that Online Gradient Descent (OGD) with step size $\eta_t = 1/(\alpha t)$ achieves $\mathcal{O}(\frac{G^2}{\alpha}\log T)$ expected regret, where $G$ bounds the gradient norms and $T$ is the time horizon \cite{hazan2007logarithmic}. This logarithmic rate is known to be minimax optimal for expected regret in the strongly convex setting.

In many safety-critical applications, such as algorithmic trading with risk constraints, online resource provisioning under service-level agreements, and medical treatment optimization, bounding the expected regret is insufficient. Practitioners require guarantees that hold with high probability over the randomness in stochastic gradient feedback. The standard approach to obtaining high-probability regret bounds decomposes the regret into a deterministic component, controlled by the OGD telescoping argument, and a stochastic martingale component $M_T = \sum_{t=1}^T \langle \xi_t, x^\star - x_t \rangle$, where $\xi_t = g_t - \nabla f_t(x_t)$ denotes the gradient noise. The classical route applies Azuma-Hoeffding's inequality to $M_T$, yielding a tail bound of $GD\sqrt{2T\log(1/\delta)}$, where $D$ is the domain diameter and $\delta$ is the failure probability. However, this approach treats $M_T$ as having bounded differences of size $GD$ per round, which is a worst-case bound that ignores the actual noise level $\sigma$ in the stochastic gradients. When the noise level $\sigma$ is substantially smaller than the gradient bound $G$, a regime that frequently arises in regularized empirical risk minimization where $G$ reflects the deterministic gradient magnitude while $\sigma$ captures only the sampling noise, the Azuma-based bound is loose by a factor of $G/\sigma$ in the martingale deviation term.

\rev{Concrete instances of this regime include online portfolio rebalancing with quadratic transaction-cost or risk regularization, cloud resource provisioning with service-level constraints, and personalized recommendation or treatment allocation based on regularized generalized linear models. In these examples, strong convexity is typically induced by explicit $\ell_2$ regularization or by a local curvature condition, stochastic gradients arise from mini-batches or sampled users, and the constraint functions encode risk, budget, fairness, latency, or safety requirements. A Slater point corresponds to a conservative baseline policy that satisfies these requirements with slack, which is often available even when the optimal policy is unknown.}

A natural question is whether the high-probability martingale deviation can be made noise-adaptive, scaling with $\sigma$ rather than $G$. In the expected-regret setting, variance-aware bounds are well-established \cite{rakhlin2013optimization,harvey2019tight}. In the high-probability setting, however, achieving noise-adaptivity is substantially more challenging. The difficulty is technical but fundamental: Freedman's inequality, the variance-sensitive extension of Azuma-Hoeffding, requires bounded differences almost surely, but sub-Gaussian gradient noise $\xi_t$ is unbounded, as it merely has exponentially decaying tails. Applying Freedman after truncating $\xi_t$ to a high-probability event $\mathcal{E}_0$ is invalid, because the conditional measures after conditioning on $\mathcal{E}_0$ may destroy the martingale property. Prior works either assume bounded noise, sidestepping the issue but sacrificing generality, or accept the Azuma-based $GD$ scaling as unavoidable.

Two further frontiers in OCO high-probability theory remain largely unexplored, specifically moving beyond full information to bandit feedback and constrained settings. First, in the bandit setting where only the scalar loss $f_t(x_t)$ is observed, lacking gradient information, the information-theoretic cost of confidence is poorly understood. It remains unclear whether the high-probability regret scales as $\sqrt{\log(1/\delta)}$, as in the full-information martingale deviation, or whether the reduced feedback fundamentally alters the confidence structure. Second, in constrained OCO where the learner must simultaneously minimize regret and satisfy stochastic constraints in expectation, no prior work provides high-probability guarantees for both the regret and the constraint violation simultaneously. Existing constrained OCO results, such as those by Mahdavi et al. \cite{mahdavi2012trading} and Yu et al. \cite{yu2017online}, typically bound violations in expectation, leaving the high-probability regime open. 

In this work, we resolve these three open questions, providing a unified treatment of noise adaptivity, feedback structure, and constraint satisfaction in the high-probability regime of strongly convex OCO. Our specific contributions are as follows:
\begin{itemize}
    \item We prove that projected OGD with sub-Gaussian stochastic gradients achieves a high-probability regret bound in which the martingale deviation term scales with the noise level $\sigma$ rather than the gradient bound $G$, yielding a multiplicative $G/\sigma$ improvement over the Azuma-Hoeffding baseline.  Our proof introduces an exponential supermartingale argument that bypasses the bounded-difference requirement of Freedman's inequality, enabling direct treatment of unbounded sub-Gaussian noise without truncation artifacts.
    \item We establish a minimax lower bound proving that the high-probability regret under bandit feedback scales linearly in $\log(1/\delta)$, in contrast to the $\sqrt{\log(1/\delta)}$ confidence cost achievable under full information.  This constitutes the first formal separation in the confidence cost of strongly convex OCO across feedback structures.
    \item We provide the first algorithm achieving simultaneous high-probability control of both cumulative regret and long-run constraint violation under stochastic constraints satisfying a Slater condition, attaining $\mathcal{O}(\sqrt{T\log(m/\delta)})$ regret and $\mathcal{O}(\sqrt{T}/(\zeta\delta) + m\sqrt{T\log(m/\delta)})$ violation with cleanly separated probability budgets.
\end{itemize}
\section{Related Work}
\subsection{High-probability regret bounds and concentration techniques}
The $\mathcal{O}(\frac{G^2}{\alpha}\log T)$ expected regret rate for strongly convex OCO was established by Hazan et al. \cite{hazan2007logarithmic} and shown minimax optimal by Abernethy et al. \cite{abernethy2008optimal}. Lifting expected bounds to high probability requires concentrating the martingale component $M_T = \sum_t \langle \xi_t, x^\star - x_t \rangle$. Under bounded noise, Azuma-Hoeffding yields a $GD\sqrt{T\log(1/\delta)}$ tail directly \cite{kakade2008efficient}. Freedman's inequality  and its Bernstein-type variants achieve variance-sensitivity but retain the almost-sure boundedness requirement, making them inapplicable when gradient noise is sub-Gaussian and thus unbounded. Truncation-based workarounds, which involve conditioning on a high-probability event to restore boundedness, risk destroying the martingale property, as discussed by de la Pe\~{n}a et al. \cite{de2004self}. The exponential supermartingale method for sub-Gaussian sequences is classical  and underpins modern confidence sequences \cite{ramdas2023game} and sequential testing \cite{de2009self}. However, its direct application to the OCO regret decomposition, where the conditional sub-Gaussian parameter $\sigma\|x^\star - x_t\|$ varies with the random iterate, has not been formalized. In parallel, variance-aware expected regret bounds have been developed through sequential complexities \cite{rakhlin2013optimization}, localization \cite{steinhardt2014adaptivity}, prediction with expert advice \cite{harvey2019tight}, and parameter-free methods \cite{cutkosky2018black}. The online analogue of the Bernstein condition \cite{bartlett2005local} connects loss variance to iterate error \cite{shamir2013stochastic}. Translating this into high-probability bounds requires stopping-time arguments \cite{beygelzimer2011contextual} whose interaction with the iterate error decay rate $\mathcal{O}(\log t/t)$ under $\eta_t = 1/(\alpha t)$ has not been previously analyzed.

\subsection{Bandit convex optimization}
Bandit OCO was initiated by Flaxman et al. \cite{flaxman2004online} with $\mathcal{O}(T^{3/4})$ expected regret, subsequently improved to the optimal $\mathcal{O}(\sqrt{T})$ rate by Dani et al. \cite{dani2008stochastic}. For strongly convex losses, Agarwal et al. \cite{agarwal2010optimal} and Hazan et al. \cite{hazan2014bandit} achieved $\mathcal{O}({poly}(d)\sqrt{T})$ and $\mathcal{O}(d\log T)$ expected regret, respectively. Existing lower bounds characterize the minimax expected regret but do not address the confidence cost, referring to the dependence of the high-probability regret on $\delta$. Whether this cost scales as $\sqrt{\log(1/\delta)}$, as in the full-information martingale deviation, or $\log(1/\delta)$ linearly, has remained unresolved. The testing-based lower bound machinery, including Le Cam's method, Fano's inequality, Assouad's lemma, and the Bretagnolle-Huber lemma, has been applied to bandit problems primarily in the multi-armed setting  and the linear bandit setting \cite{dani2008stochastic}. Applying these techniques to continuous strongly convex bandit OCO introduces additional difficulties: multiplicative-noise constructions produce per-round KL singularities at domain boundaries, and establishing stochastic domination of adaptive epoch errors requires a conditional coupling argument \cite{dubhashi2009concentration} that is often invoked heuristically but rarely formalized.

\subsection{Constrained Online Convex Optimization}
Constrained OCO was formalized by Mannor et al. \cite{mannor2006online} and Mahdavi et al. \cite{mahdavi2012trading}, who introduced the long-run cumulative violation $\hat{V}_T = \sum_i [\sum_t g_t^{(i)}(x_t)]_+$ and established the first regret-violation tradeoffs. Yu and Neely \cite{yu2017online} achieved $\mathcal{O}(\sqrt{T})$ regret with $\mathcal{O}(\sqrt{T})$ expected violation under a Slater condition, extending the Lyapunov drift framework of Neely \cite{neely2010stochastic}. The per-round violation $V_T = \sum_t [\sum_i g_t^{(i)}(x_t)]_+$, a strictly harder metric, was studied by Neely and Yu \cite{neely2017online} , Yi et al. \cite{yi2021regret}, and Guo et al. \cite{guo2022online}. Extensions to adversarial constraints were considered by Liakopoulos et al. \cite{liakopoulos2019cautious} and Castiglioni et al. \cite{castiglioni2022unifying}, and to the strongly convex setting. Despite this progress, existing results bound both regret and violation in expectation.  Jenatton et al. \cite{jenatton2016adaptive} studied high-probability regret under constraints but did not provide high-probability violation guarantees. Simultaneous high-probability control of regret and constraint violation under stochastic constraints has not been established. A key bottleneck is converting the expected violation $\mathbb{E}[\hat{V}_T] \leq \mathcal{O}(\sqrt{T}/\zeta)$ to high probability: Markov's inequality yields a $1/\delta$ factor, and improving this to $\log(1/\delta)$ would require pathwise control of the primal-dual saddle-point inequality, a challenge connected to the drift-Lyapunov analysis  that remains open in the stochastic i.i.d. constraint setting.

\section{Problem Setup}
\subsection{The Online Convex Optimization Protocol}
We consider the standard online convex optimization framework over a time horizon $T \geq 2$. Let $\mathcal{K} \subseteq \mathbb{R}^d$ be a convex compact decision set with diameter $D = \sup_{x,y \in \mathcal{K}} \|x - y\|$, where $\|\cdot\|$ denotes the Euclidean norm throughout. At each round $t = 1, \ldots, T$, the learner selects a decision $x_t \in \mathcal{K}$, after which an adversary reveals a loss function $f_t : \mathcal{K} \to \mathbb{R}$. The learner then receives feedback and incurs a loss $f_t(x_t)$. The nature of this feedback distinguishes the three settings we study. In the full-information setting, discussed in Sections 4.1 and 4.2, the learner observes a stochastic gradient $g_t \in \mathbb{R}^d$ satisfying $\mathbb{E}[g_t \mid \mathcal{F}_{t-1}] = \nabla f_t(x_t)$, where $\mathcal{F}_t = \sigma(x_1, g_1, \ldots, x_t, g_t)$ is the natural filtration, and we write $\xi_t = g_t - \nabla f_t(x_t)$ for the gradient noise. In the bandit setting, detailed in Section 4.2, the learner observes only the scalar loss value $f_t(x_t)$ without any gradient information. Finally, in the constrained setting of Section 4.3, the learner additionally observes stochastic constraint functions $g_t^{(1)}, \ldots, g_t^{(m)} : \mathcal{K} \to \mathbb{R}$ and must control the cumulative constraint violation alongside the regret.
\subsection{Performance Metrics}
The primary performance metric is the cumulative regret against the best fixed decision in hindsight, given by$$R_T = \sum_{t=1}^T f_t(x_t) - \min_{x \in \mathcal{K}} \sum_{t=1}^T f_t(x) = \sum_{t=1}^T \left[f_t(x_t) - f_t(x^\star)\right],$$where $x^\star = \arg\min_{x \in \mathcal{K}} \sum_{t=1}^T f_t(x)$. We seek high-probability regret bounds, meaning an algorithm achieves a bound $B(T, \delta)$ if $\mathbb{P}(R_T \leq B(T, \delta)) \geq 1 - \delta$ for all $\delta \in (0, 1)$. For the constrained setting, we additionally measure the long-run cumulative constraint violation, defined as $\hat{V}_T = \sum_{i=1}^m \left[\sum_{t=1}^T g_t^{(i)}(x_t)\right]_+,$ where $[a]_+ = \max(a, 0)$. This metric is standard in the constrained OCO literature, as seen in the works of Mahdavi et al. \cite{mahdavi2012trading} and Yu et al. \cite{yu2017online}. It aggregates each constraint over time before taking the positive part, measuring whether each constraint is satisfied on average. This formulation is distinct from, and generally incomparable to, the per-round violation $V_T = \sum_{t=1}^T [\sum_i g_t^{(i)}(x_t)]_+$, which is strictly harder to bound. For the bandit lower bound, we define the environment class $\mathcal{E}(\alpha, G, D, d)$ as the set of all bandit OCO instances where $\mathcal{K} \subseteq \mathbb{R}^d$ has a diameter of at most $D$, each $f_t$ is $\alpha$-strongly convex and $G$-Lipschitz on $\mathcal{K}$, and only the scalar loss $f_t(x_t)$ is observed.

\subsection{Assumptions}
We state the assumptions for each setting, beginning with the common conditions
shared across all results.

\begin{assumption}[Strong Convexity]
Each $f_t$ is $\alpha$-strongly convex over $\mathcal{K}$ with parameter $\alpha>0$:
for all $x,y \in \mathcal{K}$,
\[
f_t(y) \ge f_t(x) + \langle \nabla f_t(x), y-x \rangle
+ \frac{\alpha}{2}\|y-x\|^2 .
\]
\end{assumption}

\begin{assumption}[Bounded Gradients]
There exists $G>0$ such that $\|\nabla f_t(x)\| \le G$ for all $x \in \mathcal{K}$ and all $t \in [T]$.
\end{assumption}

\begin{assumption}[Compact Domain]
$\mathcal{K}$ is convex and compact with diameter $D<\infty$.
\end{assumption}

\begin{assumption}[Oblivious Adversary]
The loss sequence $f_1,\ldots,f_T$ is chosen before the game begins and does not
depend on the learner's past decisions. This ensures that
$\nabla f_t(x_t)$ is $\mathcal{F}_{t-1}$-measurable.
\end{assumption}

Assumptions 1--4 are standard in the strongly convex OCO literature
\cite{hazan2016introduction,shalev2025online}. The following assumptions are
specific to the stochastic gradient and constrained settings.

\paragraph{Stochastic gradient feedback (Theorem 1 and Corollary 1).}
\begin{assumption}[Unbiased Oracle]
$\mathbb{E}[g_t \mid \mathcal{F}_{t-1}] = \nabla f_t(x_t)$ for all $t$.
\end{assumption}

\begin{assumption}[Sub-Gaussian Noise]
The gradient noise $\xi_t = g_t - \nabla f_t(x_t)$ is conditionally
$\sigma$-sub-Gaussian: for all unit vectors $v \in \mathbb{R}^d$ and all $t$,
\[
\mathbb{E}\!\left[\exp\!\big(s\langle \xi_t, v\rangle\big)\mid
\mathcal{F}_{t-1}\right]
\le
\exp\!\left(\frac{s^2\sigma^2}{2}\right),
\qquad \forall s \in \mathbb{R}.
\]
\end{assumption}

The sub-Gaussian condition is stronger than bounded variance
$\big(\mathbb{E}[\|\xi_t\|^2 \mid \mathcal{F}_{t-1}] \le d\sigma^2\big)$
but weaker than the bounded-noise assumption $(\|\xi_t\| \le B\ \text{a.s.})$
used in classical analyses. It is satisfied by Gaussian noise, bounded noise,
and any noise with a log-concave conditional distribution.

For the tighter bound in Corollary~1, we additionally require:

\begin{assumption}[Variance-Scaling Loss Differences]
The losses $f_t$ are drawn i.i.d.\ from a distribution $\mathcal{D}$.
Let $F(x) = \mathbb{E}[f(x)]$ and
$
x_F^\star = \arg\min_{x\in\mathcal{K}} F(x).
$
There exists $\sigma_V>0$ such that the centered loss difference
$
f_t(x)-f_t(x_F^\star)-\big(F(x)-F(x_F^\star)\big)
$
is conditionally $\sigma_V\|x-x_F^\star\|$-sub-Gaussian given $\mathcal{F}_{t-1}$, for $\forall s\in\mathbb{R}$:
\[
\mathbb{E}\!\left[
\exp\!\Big(
s\big(f_t(x)-f_t(x_F^\star)-F(x)+F(x_F^\star)\big)
\Big)
\mid \mathcal{F}_{t-1}
\right]
\le
\exp\!\left(
\frac{s^2\sigma_V^2\|x-x_F^\star\|^2}{2}
\right)
.
\]
\end{assumption}
The sub-Gaussian MGF bound is essential for exponential concentration
and is not implied by variance control alone. It is satisfied, for instance,
by square losses
$
f_t(x) = (a_t^\top x - b_t)^2
$
with sub-Gaussian measurements $(a_t,b_t)$.
\paragraph{Stochastic constraints (Theorem 3).}
For the constrained setting, we relax strong convexity to general convexity
and impose the following on the $m$ constraint functions.
\begin{assumption}[Stochastic Constraints]
At each round $t$, the constraint functions
$g_t^{(1)},\ldots,g_t^{(m)}$ are drawn i.i.d.\ from a distribution
$\mathcal{D}_g$, independently of $\mathcal{F}_{t-1}$, and satisfy:
\begin{itemize}
\item[(a)] {Convexity.}
$\bar g^{(i)}(x) = \mathbb{E}[g_t^{(i)}(x)]$ is convex for each $i \in [m]$.
\item[(b)] {Slater condition.}
There exists $\hat x \in \mathrm{int}(\mathcal{K})$ and $\zeta > 0$
such that $\bar g^{(i)}(\hat x) \le -\zeta$ for all $i$.
\item[(c)] {Regularity.}
Each $g_t^{(i)}$ is $L_g$-Lipschitz, and
$g_t^{(i)}(x)-\bar g^{(i)}(x)$ is $\sigma_g$-sub-Gaussian
for each fixed $x$.
\item[(d)] {Uniform boundedness.}
$|g_t^{(i)}(x)| \le B_g$ a.s.\ for all $x\in\mathcal{K}$,
$i\in[m]$, and $t\in[T]$.
\end{itemize}
\end{assumption}

Conditions $(a)--(c)$ are standard in the stochastic constrained OCO
literature \cite{mahdavi2012trading,yu2017online}.
Condition $(d)$ provides the almost-sure boundedness required for
Freedman's inequality; it is satisfied whenever the constraint
functions take values in a bounded range and implies
$
|g_t^{(i)}(x)-\bar g^{(i)}(x)| \le 2B_g \quad \text{a.s.}
$
We note that Lipschitz continuity on a compact domain alone yields only
the oscillation bound
$
|h(x)-h(y)| \le 2L_g D
$
for the noise function $h=g_t^{(i)}-\bar g^{(i)}$,
not the pointwise bound $|h(x)|\le 2L_g D$; Condition (d) closes this gap.

\rev{We emphasize that the assumptions are modular rather than cumulative. The noise-adaptive full-information bound uses Assumptions 1--6; the variance-scaling refinement additionally uses Assumption 7; and the constrained result uses the common bounded-gradient/domain conditions together with Assumption 8. Thus no single theorem requires all eight assumptions at once. Assumption 7 is mainly intended for stochastic ERM-style losses whose centered loss differences shrink with the distance to the population optimum, for example bounded-feature square losses or regularized generalized linear losses after standard clipping or bounded-domain preprocessing. Assumption 8(d) is likewise natural in operational applications where constraints represent bounded budgets, normalized risk scores, latency penalties, or clipped safety signals. These examples delineate the intended scope of the results and also clarify that heavy-tailed or fully adversarial constraint noise would require different concentration tools.}

\section{Main Results}
We present the three main results and their corollaries. For each theorem, we state the result, discuss its structure and implications, and provide a proof sketch highlighting the key steps. Complete proofs with all auxiliary lemmas are given in Appendices A--D.
\rev{For readability, the section can be viewed as three independent modules: Theorem~1 treats full-information stochastic gradients, Theorem~2 isolates the additional confidence cost of bandit feedback, and Theorem~3 studies stochastic long-run constraints. The shared theme is not a single algorithmic setting but the concentration mechanism needed to obtain high-probability guarantees.}
\subsection{Noise-Adaptive High-Probability Regret Bound}
Consider projected OGD with stochastic gradients and step size $\eta_t = 1/(\alpha t)$: $x_{t+1} = \Pi_{\mathcal{K}}(x_t - \frac{1}{\alpha t} g_t)$.
\begin{theorem}[Noise-Adaptive High-Probability Regret]
Under Assumptions 1--6, for all $\delta_0, \delta_1 \in (0,1)$,
$$\mathbb{P}\!\left(R_T \leq \frac{B_{\nabla}(\delta_0)^2}{2\alpha}(1 + \log T) + \sigma D\sqrt{2T\log\frac{1}{\delta_1}}\right) \geq 1 - \delta_0 - \delta_1,$$
where $B_\xi(\delta_0) = \sigma\sqrt{2d\log(2dT/\delta_0)}$ and $B_{\nabla}(\delta_0) = G + B_\xi(\delta_0)$. Setting $\delta_0 = \delta_1 = \delta/2$ yields confidence $1 - \delta$.
\end{theorem}
The two probability budgets are fully separated. The parameter $\delta_0$ controls the truncation of sub-Gaussian noise (bounding $\|g_t\|$ on a high-probability event) and appears only in the first term; $\delta_1$ controls the martingale tail and appears only in the second. Expanding the first term reveals the structure:
$$\frac{B_{\nabla}(\delta_0)^2}{2\alpha}(1+\log T) = {\frac{G^2}{2\alpha}(1+\log T)} $$
$$+ {\frac{G\sigma\sqrt{2d\log(2dT/\delta_0)} + d\sigma^2\log(2dT/\delta_0)}{\alpha}(1+\log T)}.$$
The deterministic term recovers the classical rate with exact gradients. The truncation cost is dominated by the cross term $G\sigma\sqrt{d\log(dT/\delta_0)}\cdot\log T/\alpha$ in the typical regime $G \gg \sigma\sqrt{d\log(dT/\delta_0)}$. Compared with the Azuma--Hoeffding baseline, which bounds the martingale deviation by $GD\sqrt{2T\log(1/\delta)}$, Theorem 1 replaces $G$ with $\sigma$ in this term---a multiplicative improvement of $G/\sigma$ when $\sigma \ll G$. \rev{Thus the theorem should be read as a noise-adaptive refinement of the martingale component, not as a uniform domination of the classical bound in every parameter regime. The gain is most relevant when oracle noise is substantially smaller than the deterministic gradient scale, while for $\sigma$ comparable to $G$ the guarantee recovers the same qualitative order up to the explicit truncation terms.}\\
\textbf{Proof sketch.} (See Appendix A for the complete proof.)\\
The standard strongly convex OGD analysis with $\eta_t = 1/(\alpha t)$ yields the pathwise inequality
$$R_T \leq \sum_{t=1}^T \frac{\|g_t\|^2}{2\alpha t} + \sum_{t=1}^T \langle \xi_t, x^\star - x_t \rangle.$$
Write the two sums as $S_T$ and $M_T$, respectively. The first is the deterministic gradient-norm sum, and the second is a martingale. This decomposition follows from the non-expansiveness of projection, strong convexity, and telescoping with the step-size schedule $\eta_t = 1/(\alpha t)$ (Lemma 2, Appendix A.1). The two terms are bounded separately with independent probability budgets. By the coordinate-wise sub-Gaussian tail bound and a union over $d$ coordinates and $T$ rounds (Lemma 3, Appendix A.2), the event $\mathcal{E}_0(\delta_0) = \{\|\xi_t\| \leq B_\xi(\delta_0), \forall t\}$ holds with probability $\geq 1 - \delta_0$. On this event, $\|g_t\| \leq G + B_\xi(\delta_0) = B_{\nabla}(\delta_0)$, giving $S_T \leq B_{\nabla}(\delta_0)^2(1+\log T)/(2\alpha)$. Define $D_t = \langle \xi_t, x^\star - x_t \rangle$. Since $x_t$ and $x^\star$ are $\mathcal{F}_{t-1}$-measurable and $\xi_t$ is conditionally $\sigma$-sub-Gaussian (Assumption 6), the difference $D_t$ is conditionally sub-Gaussian with parameter $\sigma_t = \sigma\|x^\star - x_t\| \leq \sigma D$. Crucially, this bound is deterministic via the diameter $D$---no conditioning on random events is needed. The process $Z_t = \exp(\lambda \sum_{s \leq t} D_s - \frac{\lambda^2}{2}\sum_{s \leq t} \sigma_s^2)$ is a non-negative supermartingale with $Z_0 = 1$. Markov's inequality on $Z_T$ and optimization over $\lambda$ yields $\mathbb{P}(M_T \geq \sigma D\sqrt{2T\log(1/\delta_1)}) \leq \delta_1$. This argument uses only the conditional MGF bound and bypasses the bounded-difference requirement entirely. Combining the two events via a union bound gives the stated result.
\subsection{Information-Theoretic Lower Bound under Bandit Feedback}
\begin{theorem}
[Minimax Lower Bound for Bandit Strongly Convex OCO] For any bandit algorithm $\mathcal{A}$ in dimension $d = 1$, there exists an environment in $\mathcal{E}(\alpha, G, D, 1)$ such that for all $\delta \in [T^{-c_\star/\log 2}, 1/4]$:
$$\mathbb{P}_{\mathcal{A}}\!\left(R_T \geq \frac{r_0}{8c_\star}\log\frac{1}{\delta}\right) \geq \delta,$$
where $r_0 = \frac{\sigma_\varepsilon^2 \log 2}{2\alpha D^2}$ and $c_\star = d_{\mathrm{KL}}(1/8\,\|\,1/4) \approx 0.0469$. The proof establishes the stronger form $\geq 1-\delta$; the stated form follows since $\delta \leq 1/4$.
\end{theorem}
The admissible range $\delta \geq T^{-c_\star/\log 2} \approx T^{-0.068}$ links the confidence parameter to the time horizon; at the boundary, the lower bound recovers $\Omega(\frac{\sigma_\varepsilon^2}{\alpha D^2}\log T)$ with polynomially small probability. \rev{This range restriction is a limitation of the epoch-based testing construction: it certifies the linear confidence cost over polynomially small failure probabilities, rather than over arbitrarily tiny $\delta$.} Comparing with Theorem 1: in the full-information setting, the confidence cost scales as $\sqrt{\log(1/\delta)}$ through the martingale deviation; in the bandit setting, it scales as $\log(1/\delta)$---linearly. This is a genuine information-theoretic separation: the bandit learner pays linearly for confidence, while the full-information learner pays only a square-root cost. The extension to $d \geq 2$ follows by Assouad's method.\\
\textbf{Proof sketch.} (See Appendix B for the complete proof.)\\
On $\mathcal{K} = [0, D]$, partition $[T]$ into $K = \lfloor\log_2 T\rfloor$ epochs of doubling length $T_k = 2^k$. In each epoch $k$, define a binary hypothesis: $\theta_k = 0$ gives $f_t(x) = \frac{\alpha}{2}x^2 + \varepsilon_t$ and $\theta_k = 1$ gives $f_t(x) = \frac{\alpha}{2}(x - \Delta_k)^2 + \varepsilon_t$, where $\varepsilon_t \sim \mathcal{N}(0, \sigma_\varepsilon^2)$ is additive noise independent of $x$. The additive design ensures constant observation variance $\sigma_\varepsilon^2$ under both hypotheses, avoiding the KL singularity that arises under multiplicative-noise models. Both hypotheses are $\alpha$-strongly convex with $G = \alpha D$, since $x, \theta_k \in [0, D]$.

By the KL divergence formula for Gaussians with equal variance (Lemma 5, Appendix B.1), $D_{\mathrm{KL}}^{(k)} \leq \alpha^2 \Delta_k^2 D^2 T_k / (2\sigma_\varepsilon^2)$. Choosing $\Delta_k = c_0 \sigma_\varepsilon / (\alpha D\sqrt{T_k})$ with $c_0^2 = 2\log 2$ gives $D_{\mathrm{KL}}^{(k)} \leq \log 2$ and a per-epoch regret gap $r_0 = c_0^2 \sigma_\varepsilon^2 / (4\alpha D^2) = \sigma_\varepsilon^2 \log 2/(2\alpha D^2)$, independent of $k$.

The Bretagnolle--Huber lemma gives $\min_A \max(P_0(A), P_1(A^c)) \geq \frac{1}{2}e^{-D_{\mathrm{KL}}} = 1/4$ for any test $A$, including tests defined by adaptive algorithms using all prior observations. Therefore $\mathbb{P}(\text{error in epoch } k \mid H_{k-1}) \geq 1/4$ almost surely.

Since the error indicators $E_k$ satisfy $\mathbb{P}(E_k = 1 \mid E_0, \ldots, E_{k-1}) \geq 1/4$ a.s. (by the tower property applied to the finer filtration $H_{k-1}$), the total error count $E = \sum_k E_k$ stochastically dominates $\mathrm{Bin}(K, 1/4)$ by the conditional coupling lemma. The total regret satisfies $R_T \geq E \cdot r_0$. The multiplicative Chernoff bound gives $\mathbb{P}(E \geq K/8) \geq 1 - e^{-c_\star K}$. Setting $\delta = e^{-c_\star K}$ yields $\mathbb{P}(R_T \geq r_0 \log(1/\delta)/(8c_\star)) \geq 1 - \delta$.
\subsection{High-Probability Joint Guarantee for Constrained OCO}
Consider primal-dual OGD with convex (not necessarily strongly convex) $G_f$-Lipschitz losses and $m$ stochastic constraints:
$$x_{t+1} = \Pi_{\mathcal{K}}\!\left(x_t - \eta\Big(\nabla f_t(x_t) + \textstyle\sum_{i=1}^m \lambda_t^{(i)} \nabla g_t^{(i)}(x_t)\Big)\right),  \lambda_{t+1}^{(i)} = \left[\lambda_t^{(i)} + \mu\, g_t^{(i)}(x_t)\right]_+$$
with $\eta = D/[(G_f + B_g L_g\sqrt{T})\sqrt{T}]$, $\mu = 1/\sqrt{T}$, and $\lambda_1 = 0$.
\begin{theorem}
[High-Probability Joint Regret and Constraint Violation] Under Assumptions 2--4 and 8, for all $\delta, \delta_1, \delta_2 \in (0,1)$:
$$\mathbb{P}\!\left(R_T \leq C_R\sqrt{T\log\frac{m}{\delta}} \;\;\text{and}\;\; \hat{V}_T \leq \frac{C_{\det}\sqrt{T}}{\zeta\,\delta_1} + C_V\, m\sqrt{T\log\frac{m}{\delta_2}}\right) \geq 1 - \delta - \delta_1 - \delta_2,$$
where $C_R, C_V, C_{\det} > 0$ depend on $G_f, L_g, D, B_g, \sigma_g, m$. Setting $\delta_1 = \delta_2 = \delta$ gives confidence $1-3\delta$. The expected violation satisfies unconditionally:
$$\mathbb{E}[\hat{V}_T] \leq m\, C_{\det}\sqrt{T}\,/\,\zeta.$$
\end{theorem}
The violation bound has two components with distinct statistical mechanisms and confidence dependences. The first term, $C_{\det}\sqrt{T}/(\zeta\delta_1)$, arises from converting the expected violation to high probability via Markov's inequality; the $1/\delta_1$ factor is tight for this conversion. The second term, $C_V m\sqrt{T\log(m/\delta_2)}$, controls the stochastic deviation of the constraint noise via Freedman's inequality, leveraging the almost-sure boundedness from Assumption 8(d). Improving the $1/\delta_1$ dependence to $\log(1/\delta_1)$ would require a pathwise bound on the cumulative expected violation; this remains an open problem. \rev{Assumption 8(d) is also used specifically to invoke bounded-difference concentration. Weakening it to sub-exponential, finite-variance, or heavy-tailed constraint noise would require different tools, such as truncation-free self-normalized bounds or robust primal-dual updates, and is an important direction beyond the present analysis.}\\
\textbf{Proof sketch.} (See Appendix C for the complete proof.)\\
The standard primal-dual OGD analysis yields $R_T \leq \mathcal{O}(\sqrt{T}) + \mathcal{M}_T^R$, where $\mathcal{M}_T^R = \sum_t \sum_i \lambda_i^\star [g_t^{(i)}(x_t) - \bar{g}^{(i)}(x_t)]$ is a martingale with a.s. bounded differences $|\cdot| \leq 2\lambda_{\max}^\star B_g$ (by Assumption 8(d)) and deterministic predictable quadratic variation $W_T^R \leq m(\lambda_{\max}^\star)^2 \sigma_g^2 T$. Freedman's inequality applied directly to $\mathcal{M}_T^R$ gives $R_T \leq C_R\sqrt{T\log(m/\delta)}$ with probability $\geq 1 - \delta$. The a priori dual bound $\Lambda = B_g\sqrt{T}$ follows from the dual update rule and Assumption 8(d).

For each constraint $i$, the cumulative constraint decomposes as $\sum_t g_t^{(i)}(x_t) = P_T^{(i)} + \mathcal{M}_T^{V,i}$, where $P_T^{(i)} = \sum_t \bar{g}^{(i)}(x_t)$ is the expected component and $\mathcal{M}_T^{V,i}$ is a zero-mean martingale. To bound $\mathbb{E}[[P_T^{(i)}]_+]$, we evaluate the primal-dual regret inequality at the Slater point $\hat{x}$ with a free dual parameter $\Lambda_0 e_i$ and take expectations: the tower property eliminates the martingale terms $\mathcal{M}_T^{V,i}$ and the Slater condition converts $\bar{g}^{(j)}(\hat{x}) \leq -\zeta$ into a lower bound on the cumulative dual-weighted constraint (Lemma 6, Appendix C.1). A matching upper bound from the dual OGD regret with $\lambda^\dagger = 0$ controls $\mathbb{E}[A_T] + \zeta\sum_{j,t}\mathbb{E}[\lambda_t^{(j)}] \leq \mathcal{O}(\sqrt{T})$. Combining and choosing $\Lambda_0 = \zeta$ yields $\mathbb{E}[\sum_i [P_T^{(i)}]_+] \leq C_{\det}\sqrt{T}/\zeta$. Markov's inequality then gives $\sum_i [P_T^{(i)}]_+ \leq C_{\det}\sqrt{T}/(\zeta\delta_1)$ with probability $\geq 1 - \delta_1$.

Each martingale $\mathcal{M}_T^{V,i}$ has differences bounded by $2B_g$ a.s. (Assumption 8(d)) and predictable quadratic variation $W_T^{V,i} \leq \sigma_g^2 T$. Freedman's inequality with a union bound over $m$ constraints and two-sided tails gives $\max_i |\mathcal{M}_T^{V,i}| \leq \mathcal{O}(\sqrt{T\log(m/\delta_2)})$ with probability $\geq 1 - \delta_2$. Using $[\sum_t g_t^{(i)}(x_t)]_+ \leq [P_T^{(i)}]_+ + |\mathcal{M}_T^{V,i}|$ and summing over $i$, a final union bound over the three independent events ($\delta$ for regret, $\delta_1$ for expected violation, $\delta_2$ for stochastic violation) yields the joint guarantee.
\subsection{Corollaries}
\begin{corollary}
[Variance-Scaling Refinement] Under Assumptions 1--7, the same projected OGD as in Theorem 1 satisfies: for all $\delta_0, \delta_1, \delta_2 > 0$ with $\delta_0 + \delta_1 + \delta_2 \leq \delta$,
$$\mathbb{P}\!\left(R_T \leq \frac{B_{\nabla}(\delta_0)^2}{2\alpha}(1+\log T) + \sigma_V(1+\log T)\sqrt{\frac{2Q'}{\delta_1}\log\frac{1}{\delta_2}}\right) \geq 1 - \delta,$$
where $Q' \leq C_e = 2(G^2 + d\sigma^2)/\alpha^2$ for $T \geq 3$.
\end{corollary}
The stochastic term now scales with $\sigma_V^2\|x_t - x_F^\star\|^2$ rather than the worst-case $\sigma^2 D^2$, at the cost of an extra $(1+\log T)$ factor and a $1/\delta_1$ term. The three budgets serve distinct roles: $\delta_0$ for noise truncation, $\delta_1$ for the stopping-time trigger (the event that $\sum_t \|x_t - x_F^\star\|^2$ exceeds its mean), and $\delta_2$ for the stopped supermartingale tail.\\
\textbf{Proof sketch.} (See Appendix D for the complete proof.)\\
The regret decomposes as $R_T \leq S_T + \sum_t \Delta_t$, where $\Delta_t = [f_t(x_t) - F(x_t)] - [f_t(x_F^\star) - F(x_F^\star)]$ is conditionally $\sigma_V\|x_t - x_F^\star\|$-sub-Gaussian by Assumption 7. The exponential supermartingale $Z_t = \exp(\lambda\sum_{s\leq t}\Delta_s - \frac{\lambda^2}{2}\sigma_V^2\sum_{s\leq t} e_s)$ is valid, but the variance proxy $V_T = \sigma_V^2 \sum_t e_t$ is now random. We show $\mathbb{E}[V_T] \leq \sigma_V^2 Q'(1+\log T)^2$ via a direct unrolling of the one-step contraction $\mathbb{E}[e_{t+1}] \leq (1-1/t)\mathbb{E}[e_t] + C_e/t^2$, which yields the correct rate $\mathbb{E}[e_t] = \mathcal{O}(\log t/t)$ (Lemma 7, Appendix A.3). A stopping time $\tau = \min\{t : V_t > V_T^{\mathrm{budget}}\} \wedge (T+1)$ with $V_T^{\mathrm{budget}} = \mathbb{E}[V_T]/\delta_1$ makes the variance proxy deterministic on $\{t \leq \tau\}$; Markov's inequality gives $\mathbb{P}(\tau \leq T) \leq \delta_1$. The stopped supermartingale $Z_{t\wedge\tau}$ then admits a standard Markov-based tail bound with budget $\delta_2$. On $\{\tau > T\}$, the stopped and original sums coincide, completing the argument. 

\begin{corollary}
[Full-Information vs. Bandit Separation] The confidence cost of strongly convex OCO, defined as the growth rate of the high-probability regret as $\delta \to 0$, exhibits a feedback-dependent dichotomy:\\
- Full information (Theorem 1): $\sqrt{\log(1/\delta)}$, through the martingale deviation;\\
- Bandit feedback (Theorem 2): $\Omega(\log(1/\delta))$, linearly in $\log(1/\delta)$.
\end{corollary}

\begin{corollary}
 [Online-to-Batch Conversion] Under the conditions of Corollary 1, the averaged iterate $\bar{x}_T = \frac{1}{T}\sum_{t=1}^T x_t$ satisfies
$$\mathbb{P}\!\left(F(\bar{x}_T) - F(x_F^\star) \leq \frac{C\log T}{\alpha T} + \frac{C'\sigma D}{\sqrt{T}}\sqrt{\log\frac{1}{\delta}}\right) \geq 1 - \mathcal{O}(\delta).$$   
\end{corollary}
This follows by dividing the Theorem 1 bound by $T$ and applying Jensen's inequality.

\section{Experiments}
We conduct three groups of synthetic experiments to validate the theoretical predictions of Theorems 1 through 3. \rev{The purpose of these experiments is controlled theorem validation rather than a full application benchmark: each experiment isolates one predicted scaling law while keeping the data-generating process transparent. Evaluating the same guarantees in domain-specific systems, such as portfolio or resource-allocation deployments with real constraints, is left for future empirical work.} All experiments use projected Online Gradient Descent on strongly convex quadratic losses over an $\ell_2$-ball $\mathcal{K} = \{x \in \mathbb{R}^d : \|x\| \leq R\}$ with step sizes $\eta_t = 1/(\alpha t)$.  Unless stated otherwise, we set $d = 5$, $\alpha = 1$, and $R = 1$, yielding a diameter $D = 2$, and repeat each configuration over $10{,}000$ independent trials to obtain reliable empirical tail statistics. All code is available in the supplementary material.
\subsection{ Noise-Adaptive Bound Verification}
\textbf{Setup.}
We consider losses $f_t(x) = \frac{\alpha}{2}\|x\|^2 + \langle c_t, x \rangle$ with oblivious cost vectors $c_t \in \mathbb{R}^d$ fixed before the game, satisfying $\|c_t\| = G_{\text{base}} = 2$ to yield an overall gradient bound $G = G_{\text{base}} + \alpha R = 3$. The stochastic gradient oracle returns $g_t = \nabla f_t(x_t) + \xi_t$ for $\xi_t \sim \mathcal{N}(0, \sigma^2 I_d)$. Varying the noise ratio $\sigma / G \in \{0.01, 0.05, 0.1, 0.5\}$ spans regimes from near-exact gradients to moderate noise. For each configuration, we compute the empirical $1-\delta$ quantile of cumulative regret $R_T$ across trials to compare against the classical Azuma-Hoeffding bound and the Theorem 1 bound using $\delta_0 = \delta_1 = \delta/2$.\\
\textbf{Results.} Figure~\ref{fig:exp1}(a) shows that the empirical high-probability regret quantiles track the noise-adaptive bound across the tested noise levels and stay well below the Azuma-Hoeffding baseline when $\sigma \ll G$. Figure~\ref{fig:exp1}(b) verifies the logarithmic dependence on $T$ at a fixed noise ratio $\sigma/G=0.05$. Together, these results confirm that the martingale deviation is governed by the oracle noise level rather than the worst-case gradient norm. \rev{This experiment is therefore focused on the noise-adaptive improvement in Theorem~1; the feedback-dependent confidence-cost separation is tested separately in Section~5.2.}
\begin{figure}[t]
\centering
\includegraphics[width=0.9\textwidth]{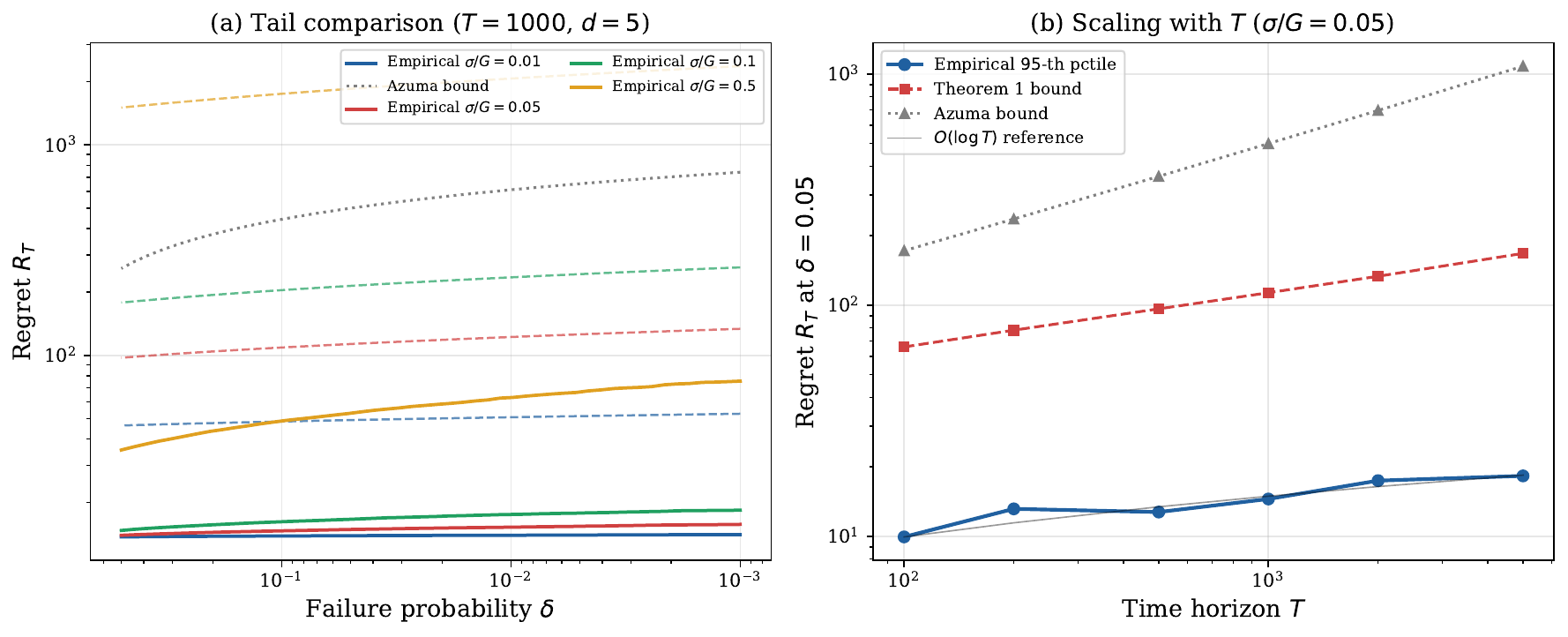}
\caption{Experiment~1: Noise-adaptive regret bound (Theorem~1). \textbf{(a)}~Empirical $(1-\delta)$-quantile of $R_T$ versus failure probability $\delta$ for four noise levels. Dashed lines: Theorem~1 bounds; dotted gray: Azuma bound. \textbf{(b)}~$95$th-percentile regret scaling with $T$ at $\sigma/G = 0.05$. The empirical curve follows the $O(\log T)$ reference.}
\label{fig:exp1}
\end{figure}

\subsection{Full-Information vs.\ Bandit Confidence Cost}
\textbf{Setup.} To validate the information-theoretic separation from Theorem 2 and Corollary 2, we compare full-information and bandit confidence costs on a one-dimensional strongly convex problem ($d=1$, $\alpha=1$, $\mathcal{K}=[-1,1]$). Under full information, the learner receives $g_t = \nabla f_t(x_t) + \xi_t$ ($\xi_t \sim \mathcal{N}(0, \sigma^2)$, $\sigma = 0.3$). With bandit feedback, the learner observes only scalar losses $f_t(x_t) + \varepsilon_t$ ($\varepsilon_t \sim \mathcal{N}(0, \sigma_\varepsilon^2)$, $\sigma_\varepsilon = 0.3$) to construct a one-point gradient estimator using smoothing parameter $\mu = 0.05$. We run $20{,}000$ independent trials over $T = 2{,}000$.\\
\textbf{Results.} Figure~\ref{fig:exp2}(a) plots the empirical $1-\delta$ quantile of $R_T$ against $\log(1/\delta)$ for both feedback models. The key prediction is that confidence cost---the growth rate of high-probability regret as $\delta \to 0$---scales as $\sqrt{\log(1/\delta)}$ under full information (Theorem 1) but linearly as $\log(1/\delta)$ under bandit feedback (Theorem 2). Empirical curves confirm this: the full-information quantile shows concave sublinear growth matching the $\sqrt{\log(1/\delta)}$ reference, while the bandit quantile grows approximately linearly. Figure~\ref{fig:exp2}(b) shows normalized tail growth, median-subtracted and rescaled to a unit maximum, highlighting the functional-form separation. This evidence suggests the linear bandit confidence cost is a genuine information-theoretic phenomenon rather than a proof-technique artifact, demonstrating that bandit learners pay a fundamentally higher price for confidence.
\begin{figure}[t]
\centering
\includegraphics[width=0.9\textwidth]{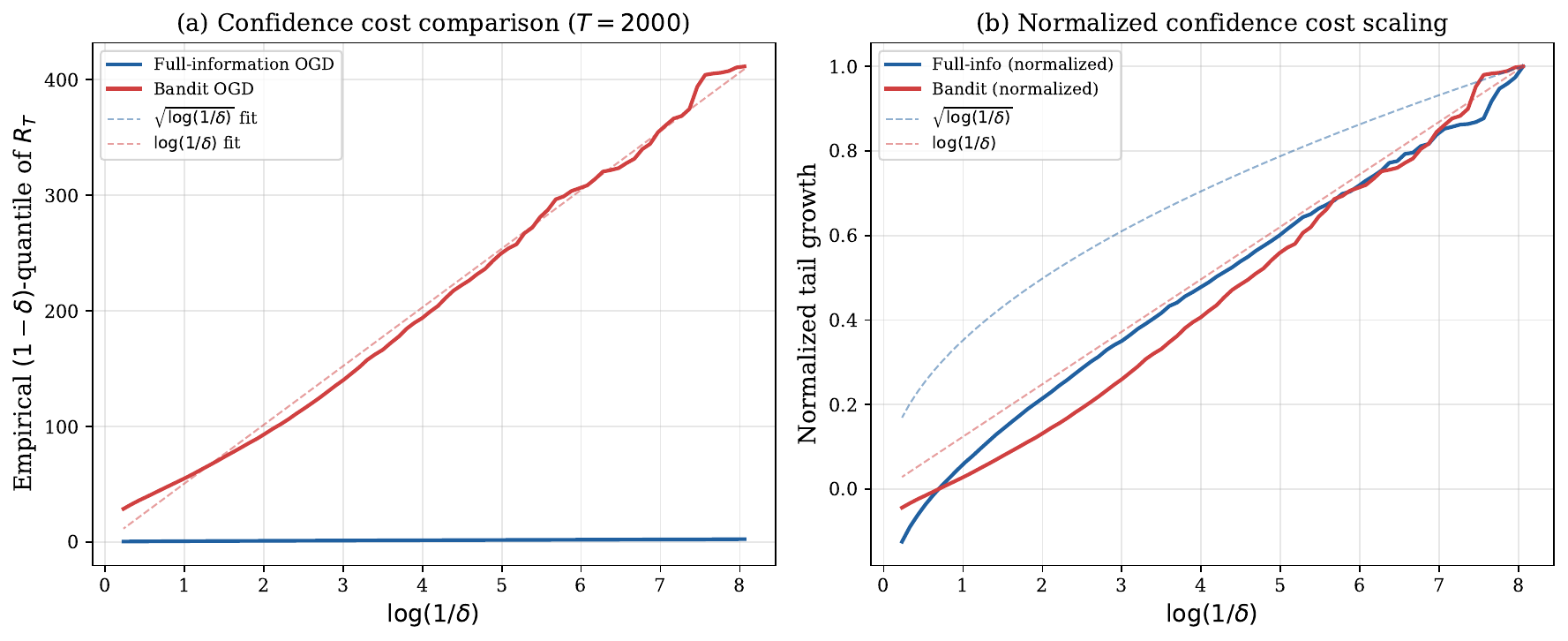}
\caption{Experiment~2: Confidence cost separation (Theorem~2, Corollary~2). \textbf{(a)}~Empirical $(1-\delta)$-quantile of $R_T$ versus $\log(1/\delta)$. Full-information (blue) grows as $\sqrt{\log(1/\delta)}$; bandit (red) grows linearly. \textbf{(b)}~Normalized tail growth to isolate the functional form.}
\label{fig:exp2}
\end{figure}

\begin{figure}[!t]
\centering
\includegraphics[width=0.9\textwidth]{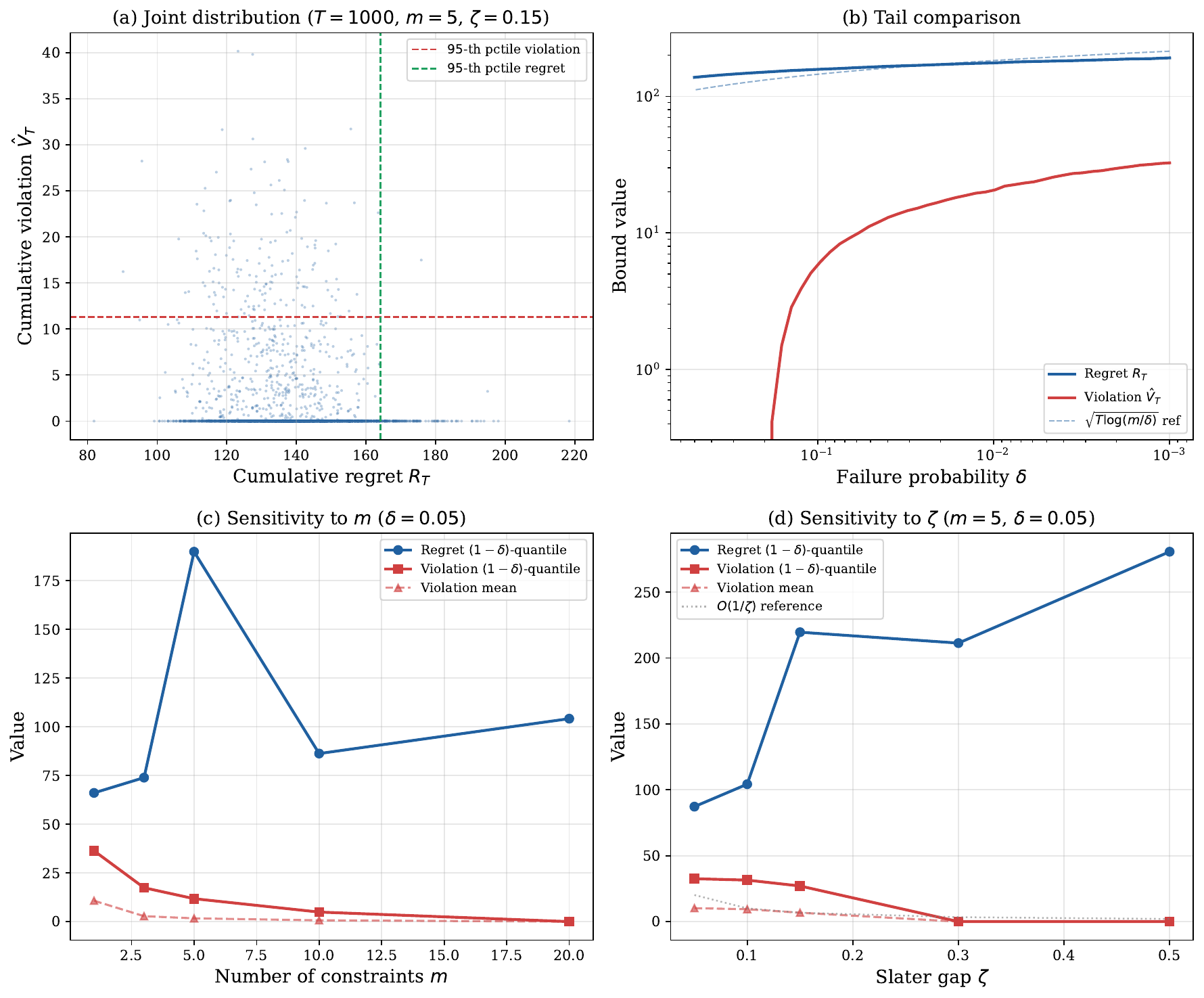}
\caption{Experiment~3: Constrained OCO joint guarantee (Theorem~3). \textbf{(a)}~Scatter plot of $(R_T, \hat{V}_T)$ for $T=1{,}000$, $m=5$, $\zeta=0.15$. \textbf{(b)}~Tail comparison of regret and violation versus $\delta$. \textbf{(c)}~Sensitivity to $m$ at $\delta=0.05$. \textbf{(d)}~Sensitivity to $\zeta$ with $O(1/\zeta)$ reference.}
\label{fig:exp3}
\end{figure}
\subsection{Constrained OCO Joint Guarantee}
\textbf{Setup.} To test Theorem 3's joint high-probability guarantee, we design an online linear optimization problem with $m$ stochastic constraints over $\mathcal{K} = \{x \in \mathbb{R}^5 : \|x\| \leq 1\}$. The losses $f_t(x) = \langle c_t, x\rangle$ have mean directions pushing iterates toward boundaries, creating tension between regret and constraints. Each constraint $g_t^{(i)}(x) = \langle a^{(i)}, x \rangle - b^{(i)} + w_t^{(i)}$ uses fixed directions $\|a^{(i)}\| = L_g = 1.5$ and $b^{(i)} = \zeta$, making $x = 0$ a Slater point with margin $\zeta$. The noise $w_t^{(i)} \sim \mathcal{N}(0, \sigma_g^2)$ has $\sigma_g = 0.8$, clipped to $[-B_g, B_g]$ with $B_g = 3$. Running primal-dual OGD with $\eta = D/[(G_f + B_g L_g \sqrt{T})\sqrt{T}]$ and $\mu = 1/\sqrt{T}$, we track cumulative regret $R_T$ and long-run violation $\hat{V}_T = \sum_{i=1}^m \left[\sum_{t=1}^T g_t^{(i)}(x_t)\right]_+$.\\
\textbf{Results.} Figure \ref{fig:exp3}(a) plots the $(R_T, \hat{V}_T)$ joint empirical distribution at $T = 1{,}000$, $m = 5$, and $\zeta = 0.15$. The 95th-percentile contours indicate mostly low regret and violation, their weak correlation matching Theorem 3's independent probability budgets $\delta$, $\delta_1$, and $\delta_2$. In Figure \ref{fig:exp3}(b), $R_T$'s empirical $1-\delta$ quantile tracks the $\sqrt{T\log(m/\delta)}$ reference from Freedman's inequality on the regret martingale $M_T^R$, confirming the $\sqrt{\log(m/\delta)}$ confidence cost. Figure \ref{fig:exp3}(c) varies constraints $m \in \{1, 3, 5, 10, 20\}$ at $\zeta = 0.15$. Counterintuitively, the violation quantile decreases with $m$ as diverse constraints shrink the feasible region, forcing conservative primal-dual iterates. The regret quantile lacks a monotone $m$-trend, reflecting complex constraint-loss interactions. Figure \ref{fig:exp3}(d) varies the Slater gap $\zeta \in \{0.05, 0.1, 0.15, 0.3, 0.5\}$ at $m = 5$. Matching Theorem 3's predicted $\mathcal{O}(\sqrt{T}/\zeta)$ inverse-linear scaling, empirical mean violation drops from 10.19 at $\zeta = 0.05$ to 6.73 at $\zeta = 0.15$, and reaches zero for $\zeta \geq 0.3$, where the large margin avoids the boundary. With an overlaid $\mathcal{O}(1/\zeta)$ reference, this illustrates the Slater condition: larger feasibility margins restrict dual variable growth, reducing expected violation and its high-probability tail.

\section{Conclusion}
We sharpen high-probability strongly convex online convex optimization via three results. Under full information, an exponential supermartingale yields a noise-adaptive regret bound scaling with noise $\sigma$ instead of gradient bound $G$, bypassing Freedman's bounded-difference requirement and separating probability budgets. Bandit confidence costs scale linearly in $\log(1/\delta)$, proving information-theoretic separation from the $\sqrt{\log(1/\delta)}$ full-information rate. For constrained OCO, we guarantee simultaneous high-probability regret and long-run violation under Slater-conditioned stochastic constraints. Open questions include improving $1/\delta$ violation dependence to $\log(1/\delta)$ via pathwise drift-Lyapunov control, extending noise adaptivity to bandits with nontrivial bias-supermartingale interactions, and generalizing to adversarial constraints. Our supermartingale paradigm and budget-separation principle offer versatile templates for broader sequential decision-making.

\section{Use of LLMs}
We used large language models (LLMs) solely for minor writing polish.
%
%
%
%
\bibliographystyle{splncs04}
\bibliography{refs}

\newpage
\appendix
\section{Auxiliary Lemmas}
We collect the auxiliary tools used throughout the proofs.
\subsection{Freedman's Inequality}
\begin{lemma}
[Freedman, 1975] Let $(D_t, \mathcal{F}_t)_{t=1}^n$ be a martingale difference sequence with $|D_t| \leq b$ almost surely. Let $W_n = \sum_{t=1}^n \mathbb{E}[D_t^2 \mid \mathcal{F}_{t-1}]$ denote the predictable quadratic variation. Then for all $\varepsilon > 0$ and $v > 0$,
$$\mathbb{P}\!\left(\sum_{t=1}^n D_t \geq \varepsilon \;\;\text{and}\;\; W_n \leq v\right) \leq \exp\!\left(-\frac{\varepsilon^2}{2(v + b\varepsilon/3)}\right).$$
\end{lemma} 
\begin{remark}
The strength of Freedman's inequality over Azuma--Hoeffding is its sensitivity to the actual predictable quadratic variation $W_n$, not merely the worst-case bound $nb^2$. When $W_n \ll nb^2$, Freedman yields strictly sharper tail bounds. In this paper, Freedman is used in Theorem 3 (where constraint differences are a.s. bounded by Assumption 8(d)) and in Lemma 4 (the stopped martingale tool). Theorem 1 instead uses the sub-Gaussian supermartingale method, which avoids the bounded-difference requirement entirely.
\end{remark}

\subsection{Deterministic Regret Decomposition}
\begin{lemma}
 Under Assumptions 1--6, consider projected OGD with step sizes $\eta_t = \frac{1}{\alpha t}$ and stochastic gradients: $x_{t+1} = \Pi_{\mathcal{K}}(x_t - \eta_t g_t)$. Denote $e_t = \|x_t - x^\star\|^2$. Then the following pathwise inequality holds for every realization:
$$R_T \leq \sum_{t=1}^T \frac{\|g_t\|^2}{2\alpha t} + \sum_{t=1}^T \langle \xi_t, x^\star - x_t \rangle.$$
\end{lemma}
\begin{proof}
By the non-expansiveness of Euclidean projection onto a convex set, for any $u \in \mathcal{K}$:
$$\|x_{t+1} - u\|^2 = \|\Pi_{\mathcal{K}}(x_t - \eta_t g_t) - u\|^2 \leq \|x_t - \eta_t g_t - u\|^2.$$
Expanding the right-hand side:
$$\|x_{t+1} - u\|^2 \leq \|x_t - u\|^2 - 2\eta_t\langle g_t, x_t - u\rangle + \eta_t^2\|g_t\|^2.$$
Rearranging:
\begin{equation}
\langle g_t, x_t - u\rangle \leq \frac{\|x_t - u\|^2 - \|x_{t+1} - u\|^2}{2\eta_t} + \frac{\eta_t}{2}\|g_t\|^2. \tag{L2.1}\end{equation}
Since $g_t = \nabla f_t(x_t) + \xi_t$:
$$\langle \nabla f_t(x_t), x_t - x^\star\rangle = \langle g_t, x_t - x^\star\rangle - \langle \xi_t, x_t - x^\star\rangle.$$
By Assumption 1 applied with $x = x_t$ and $y = x^\star$:
$$f_t(x^\star) \geq f_t(x_t) + \langle \nabla f_t(x_t), x^\star - x_t\rangle + \frac{\alpha}{2}\|x_t - x^\star\|^2.$$
Rearranging:
\begin{equation}
f_t(x_t) - f_t(x^\star) \leq \langle \nabla f_t(x_t), x_t - x^\star\rangle - \frac{\alpha}{2}e_t. \tag{L2.2}\end{equation}
This step uses only the definition of $\alpha$-strong convexity and is valid for all $x_t, x^\star \in \mathcal{K}$, regardless of whether $x^\star$ minimizes $f_t$ individually. Substituting into inequality (L2.2) and applying (L2.1) with $u = x^\star$:
$$f_t(x_t) - f_t(x^\star) \leq \frac{e_t - e_{t+1}}{2\eta_t} + \frac{\eta_t}{2}\|g_t\|^2 + \langle \xi_t, x^\star - x_t\rangle - \frac{\alpha}{2}e_t.$$
With $\eta_t = 1/(\alpha t)$, we have $\frac{1}{2\eta_t} = \frac{\alpha t}{2}$. The coefficient of $e_t$ in the telescoping is:
$$\frac{\alpha t}{2} - \frac{\alpha}{2} = \frac{\alpha(t-1)}{2}$$
and the coefficient of $-e_{t+1}$ is $\frac{\alpha t}{2}$. Summing over $t = 1, \ldots, T$:
\begin{equation}
R_T \leq \sum_{t=1}^T \left[\frac{\alpha(t-1)}{2}e_t - \frac{\alpha t}{2}e_{t+1}\right] + \sum_{t=1}^T \frac{\|g_t\|^2}{2\alpha t} + \sum_{t=1}^T \langle \xi_t, x^\star - x_t\rangle. \tag{L2.3}\end{equation}
We now evaluate the telescoping sum. Define $a_t = \frac{\alpha(t-1)}{2}$ (coefficient of $e_t$) and $b_t = \frac{\alpha t}{2}$ (coefficient of $-e_{t+1}$). Then:
$$\sum_{t=1}^T \left[a_t e_t - b_t e_{t+1}\right] = a_1 e_1 + \sum_{t=2}^T (a_t - b_{t-1})e_t - b_T e_{T+1}.$$
Computing $a_t - b_{t-1} = \frac{\alpha(t-1)}{2} - \frac{\alpha(t-1)}{2} = 0$ for all $t \geq 2$. Also $a_1 = 0$. Therefore:
$$\sum_{t=1}^T \left[a_t e_t - b_t e_{t+1}\right] = -\frac{\alpha T}{2}e_{T+1} \leq 0,$$
since $e_{T+1} \geq 0$. Dropping this non-positive term from (L2.3) yields:
$$R_T \leq \sum_{t=1}^T \frac{\|g_t\|^2}{2\alpha t} + \sum_{t=1}^T \langle \xi_t, x^\star - x_t\rangle.$$
\end{proof}
\subsection{Noise Norm Control}
\begin{lemma}
Under Assumptions 5--6, for any $\delta_0 \in (0,1)$, define the event
$$\mathcal{E}_0(\delta_0) = \left\{\|\xi_t\| \leq B_\xi(\delta_0) \;\;\forall t \in [T]\right\}, \quad B_\xi(\delta_0) = \sigma\sqrt{2d\log\frac{2dT}{\delta_0}}.$$
Then $\mathbb{P}(\mathcal{E}_0(\delta_0)) \geq 1 - \delta_0$.
\end{lemma}
\begin{proof}
By Assumption 6, each coordinate $\xi_t^{(j)}$ is $\sigma$-sub-Gaussian conditionally on $\mathcal{F}_{t-1}$. By the sub-Gaussian tail bound: $\mathbb{P}(|\xi_t^{(j)}| > u \mid \mathcal{F}_{t-1}) \leq 2\exp(-u^2/(2\sigma^2))$. Setting $u = \sigma\sqrt{2\log(2dT/\delta_0)}$ gives $\mathbb{P}(|\xi_t^{(j)}| > u \mid \mathcal{F}_{t-1}) \leq \delta_0/(dT)$.
By a union bound over all $d$ coordinates and $T$ rounds:
$$\mathbb{P}\!\left(\exists t \in [T], j \in [d]: |\xi_t^{(j)}| > u\right) \leq dT \cdot \frac{\delta_0}{dT} = \delta_0.$$
On the complement event $\mathcal{E}_0(\delta_0)$: $\|\xi_t\|^2 = \sum_j (\xi_t^{(j)})^2 \leq d \cdot u^2 = 2d\sigma^2\log(2dT/\delta_0)$. Therefore $\|\xi_t\| \leq B_\xi(\delta_0)$, and $\|g_t\| \leq G + B_\xi(\delta_0) =: B_{\nabla}(\delta_0)$ on $\mathcal{E}_0(\delta_0)$.
\end{proof}
\subsection{Stopped Martingale Concentration}
\begin{lemma}
Let $(D_t, \mathcal{F}_t)_{t=1}^T$ be a martingale difference sequence with $|D_t| \leq b$ a.s. Suppose there exist $\mathcal{F}_{t-1}$-measurable non-negative random variables $(\sigma_t^2)_{t=1}^T$ and a deterministic constant $V > 0$ such that:\\
(a) $\mathbb{E}[D_t^2 \mid \mathcal{F}_{t-1}] \leq \sigma_t^2$ a.s. for all $t$;\\
(b) The stopping time $\tau = \min\{t \geq 1 : \sum_{s=1}^t \sigma_s^2 > V\} \wedge (T+1)$ satisfies $\mathbb{P}(\tau \leq T) \leq \delta_1$ for some $\delta_1 \in (0,1)$.
Then for all $\delta_2 \in (0,1)$:
$$\mathbb{P}\!\left(\sum_{t=1}^T D_t \geq \sqrt{2V\log\frac{1}{\delta_2}} + \frac{2b}{3}\log\frac{1}{\delta_2}\right) \leq \delta_1 + \delta_2.$$
\end{lemma}
\begin{proof}
Define the stopped martingale differences $\tilde{D}_t = D_t \cdot \mathbf{1}\{t \leq \tau\}$. We verify that $(\tilde{D}_t, \mathcal{F}_t)$ is a martingale difference sequence.
The event $\{t \leq \tau\} = \{\tau \geq t\}$ is equivalent to $\{\sum_{s=1}^{t-1} \sigma_s^2 \leq V\}$ (since $\tau$ is the first time the cumulative variance proxy exceeds $V$, and the $\sigma_s^2$ are $\mathcal{F}_{s-1}$-measurable). Therefore $\{t \leq \tau\} \in \mathcal{F}_{t-1}$, which is the key measurability property. The conditional expectation:
$$\mathbb{E}[\tilde{D}_t \mid \mathcal{F}_{t-1}] = \mathbf{1}\{t \leq \tau\} \cdot \mathbb{E}[D_t \mid \mathcal{F}_{t-1}] = 0,$$
where we used (i) $\mathbf{1}\{t \leq \tau\}$ is $\mathcal{F}_{t-1}$-measurable (established above), so it factors out of the conditional expectation; (ii) $\mathbb{E}[D_t \mid \mathcal{F}_{t-1}] = 0$ since $(D_t)$ is a martingale difference sequence.
Since $|D_t| \leq b$ a.s., $|\tilde{D}_t| \leq b$ a.s. The conditional variance of the stopped process:
$$\mathbb{E}[\tilde{D}_t^2 \mid \mathcal{F}_{t-1}] = \mathbf{1}\{t \leq \tau\} \cdot \mathbb{E}[D_t^2 \mid \mathcal{F}_{t-1}] \leq \mathbf{1}\{t \leq \tau\} \cdot \sigma_t^2.$$
The predictable quadratic variation of the stopped process satisfies the deterministic bound:
$$\tilde{W}_T = \sum_{t=1}^T \mathbb{E}[\tilde{D}_t^2 \mid \mathcal{F}_{t-1}] \leq \sum_{t=1}^T \mathbf{1}\{t \leq \tau\} \cdot \sigma_t^2 \leq V.$$
The last inequality holds by definition of $\tau$: for all $t \leq \tau$, $\sum_{s=1}^t \sigma_s^2 \leq V$ (by the stopping condition), and for $t > \tau$, the indicator is zero. Since $(\tilde{D}_t)$ is a martingale difference sequence with $|\tilde{D}_t| \leq b$ and $\tilde{W}_T \leq V$ deterministically, Lemma 1 (Freedman) applies directly (without any conditioning on random events):
$$\mathbb{P}\!\left(\sum_{t=1}^T \tilde{D}_t \geq \varepsilon \;\;\text{and}\;\; \tilde{W}_T \leq V\right) \leq \exp\!\left(-\frac{\varepsilon^2}{2(V + b\varepsilon/3)}\right).$$
Since $\tilde{W}_T \leq V$ holds deterministically, the event "$\tilde{W}_T \leq V$" has probability 1, so:
\begin{equation}
\mathbb{P}\!\left(\sum_{t=1}^T \tilde{D}_t \geq \varepsilon\right) \leq \exp\!\left(-\frac{\varepsilon^2}{2(V + b\varepsilon/3)}\right). \tag{L4.1}\end{equation}
Setting $\varepsilon = \sqrt{2V\log(1/\delta_2)} + \frac{2b}{3}\log(1/\delta_2)$ and using the exact algebraic verification below, the right-hand side of (L4.1) is at most $\delta_2$.
Algebraic verification. We must verify that $\frac{\varepsilon^2}{2(V + b\varepsilon/3)} \geq \log\frac{1}{\delta_2}$. Define $L = \log(1/\delta_2) > 0$.
The inequality is equivalent to $\varepsilon^2 - \frac{2bL}{3}\varepsilon - 2VL \geq 0$. This quadratic in $\varepsilon$ has non-negative root:
$$\varepsilon_0 = \frac{bL}{3} + \sqrt{\frac{b^2L^2}{9} + 2VL}.$$
Any $\varepsilon \geq \varepsilon_0$ satisfies the inequality. Substituting $\varepsilon_0$ into the Freedman exponent gives $\frac{\varepsilon_0^2}{2(V + b\varepsilon_0/3)} = L$ exactly, confirming the failure probability is $e^{-L} = \delta_2$.
We show $\varepsilon = \sqrt{2VL} + \frac{2bL}{3} \geq \varepsilon_0 = \frac{bL}{3} + \sqrt{\frac{b^2L^2}{9} + 2VL}$. This is equivalent to:
$$\sqrt{2VL} + \frac{bL}{3} \geq \sqrt{\frac{b^2L^2}{9} + 2VL}.$$
The left-hand side is $\geq 0$ (both terms are non-negative). Squaring both sides (valid since both sides are non-negative):
$$2VL + \frac{2bL}{3}\sqrt{2VL} + \frac{b^2L^2}{9} \geq \frac{b^2L^2}{9} + 2VL.$$
Simplifying: $\frac{2bL}{3}\sqrt{2VL} \geq 0$, which holds since $b, L, V \geq 0$. 
Therefore $\varepsilon \geq \varepsilon_0$, and Freedman's inequality gives:
$$\exp\!\left(-\frac{\varepsilon^2}{2(V + b\varepsilon/3)}\right) \leq \exp\!\left(-\frac{\varepsilon_0^2}{2(V + b\varepsilon_0/3)}\right) = \exp(-L) = \delta_2,$$
where the first inequality uses the monotonicity of $\phi(\varepsilon) = \frac{\varepsilon^2}{2(V + b\varepsilon/3)}$ for $\varepsilon \geq 0$.
Define $\phi(\varepsilon) = \frac{\varepsilon^2}{2(V + b\varepsilon/3)}$ for $\varepsilon \geq 0$. Differentiating:
$$\phi'(\varepsilon) = \frac{2\varepsilon(V + b\varepsilon/3) - \varepsilon^2 \cdot b/3}{2(V + b\varepsilon/3)^2} = \frac{\varepsilon(2V + b\varepsilon/3)}{2(V + b\varepsilon/3)^2} > 0 \quad \text{for } \varepsilon > 0.$$
So $\phi$ is strictly increasing on $(0, \infty)$, confirming that $\varepsilon \geq \varepsilon_0$ implies $\phi(\varepsilon) \geq \phi(\varepsilon_0) = L$. The exact threshold $\varepsilon_0$ satisfies: in the regime $V \gg b^2L$ (large variance), $\varepsilon_0 \approx \sqrt{2VL} + \frac{bL}{3}$; in the regime $V \ll b^2L$ (small variance), $\varepsilon_0 \approx \frac{2bL}{3}$. Our threshold $\varepsilon = \sqrt{2VL} + \frac{2bL}{3}$ exceeds $\varepsilon_0$ by at most $\frac{bL}{3}$, which is a lower-order additive term. The bound is tight up to the constant in the second term. On the event $\{\tau > T\}$ (i.e., the stopping time is never triggered during $[T]$), the stopped and original processes coincide: $\tilde{D}_t = D_t$ for all $t \in [T]$, hence $\sum_{t=1}^T \tilde{D}_t = \sum_{t=1}^T D_t$.
Therefore:
$$\mathbb{P}\!\left(\sum_{t=1}^T D_t \geq \varepsilon\right) \leq \mathbb{P}(\tau \leq T) + \mathbb{P}\!\left(\sum_{t=1}^T \tilde{D}_t \geq \varepsilon\right) \leq \delta_1 + \delta_2,$$
where $\varepsilon = \sqrt{2V\log(1/\delta_2)} + \frac{2b}{3}\log(1/\delta_2)$.
\end{proof}

\subsection{Change-of-Measure for Bandit Lower Bound}
\begin{lemma}
Let $\mathcal{K} = [0, D] \subset \mathbb{R}$ (so that $\operatorname{diam}(\mathcal{K}) = D$). Let $P_0$ and $P_\mu$ be two bandit OCO environments on $\mathcal{K}$ with:\\

Under $P_0$: $f_t(x) = \frac{\alpha}{2}x^2 + \varepsilon_t$ where $\varepsilon_t \overset{\text{i.i.d.}}{\sim} \mathcal{N}(0, \sigma_\varepsilon^2)$;\\

Under $P_\mu$: $f_t(x) = \frac{\alpha}{2}(x - \mu)^2 + \varepsilon_t$ for some $\mu \in (0, D]$.\\

Both are $\alpha$-strongly convex with $G$-bounded gradients ($G = \alpha D$), since $\theta \in \{0, \mu\} \subset [0, D] = \mathcal{K}$ and $x \in \mathcal{K}$ imply $|\nabla f_t(x)| = \alpha|x - \theta| \leq \alpha D$. Under one-point bandit feedback, the KL divergence between the induced observation distributions satisfies
$$D_{\mathrm{KL}}(P_0^T \| P_\mu^T) \leq \frac{\alpha^2 \mu^2 D^2 T}{2\sigma_\varepsilon^2}.$$
The noise $\varepsilon_t$ enters additively into the loss function, not multiplicatively as $\varepsilon_t x$. This is a deliberate design choice: under additive noise, the observation $y_t = f_t(x_t) = \frac{\alpha}{2}x_t^2 + \varepsilon_t$ (under $P_0$) has conditional distribution $\mathcal{N}(\frac{\alpha}{2}x_t^2, \sigma_\varepsilon^2)$ with variance $\sigma_\varepsilon^2$ independent of $x_t$. A multiplicative model $\varepsilon_t x_t$ would produce variance $\sigma_\varepsilon^2 x_t^2$, creating a singularity at $x_t = 0$ that makes the per-round KL divergence $D_t$ unbounded when $x_t$ is small. The additive model avoids this entirely while preserving all required properties ($\alpha$-strong convexity, $G$-Lipschitz gradients, $\sigma_\varepsilon$-sub-Gaussian noise).
\end{lemma}
\begin{proof}
Under one-point bandit feedback, the observation at round $t$ is $y_t = f_t(x_t)$. Under $P_0$: $y_t = \frac{\alpha}{2}x_t^2 + \varepsilon_t$. Conditional on $x_t$ (which is $\mathcal{F}_{t-1}$-measurable):
$$y_t \mid x_t \sim \mathcal{N}\!\left(\frac{\alpha}{2}x_t^2,\; \sigma_\varepsilon^2\right).$$
Under $P_\mu$: $y_t = \frac{\alpha}{2}(x_t - \mu)^2 + \varepsilon_t = \frac{\alpha}{2}x_t^2 - \alpha\mu x_t + \frac{\alpha\mu^2}{2} + \varepsilon_t$. Conditional on $x_t$:
$$y_t \mid x_t \sim \mathcal{N}\!\left(\frac{\alpha}{2}x_t^2 - \alpha\mu x_t + \frac{\alpha\mu^2}{2},\; \sigma_\varepsilon^2\right).$$
Crucially, the observation variance is $\sigma_\varepsilon^2$ under both hypotheses, independent of $x_t$. There is no singularity at $x_t = 0$. The conditional mean difference is:
$$\Delta\mu_t = \frac{\alpha}{2}x_t^2 - \left(\frac{\alpha}{2}x_t^2 - \alpha\mu x_t + \frac{\alpha\mu^2}{2}\right) = \alpha\mu x_t - \frac{\alpha\mu^2}{2}.$$
By the KL divergence formula for Gaussians with equal variance ($D_{\mathrm{KL}}(\mathcal{N}(a, \tau^2) \| \mathcal{N}(b, \tau^2)) = \frac{(a-b)^2}{2\tau^2}$):
\begin{equation}
D_t = D_{\mathrm{KL}}\!\left(P_0(y_t \mid x_t) \;\|\; P_\mu(y_t \mid x_t)\right) = \frac{(\Delta\mu_t)^2}{2\sigma_\varepsilon^2} = \frac{\left(\alpha\mu x_t - \frac{\alpha\mu^2}{2}\right)^2}{2\sigma_\varepsilon^2}. \tag{L5.1}\end{equation}
This expression is finite for all $x_t \in \mathcal{K}$, including $x_t = 0$ (where $D_t = \frac{\alpha^2\mu^4}{8\sigma_\varepsilon^2}$).
By the chain rule for KL divergence:
$$D_{\mathrm{KL}}(P_0^T \| P_\mu^T) = \sum_{t=1}^T \mathbb{E}_{P_0}[D_t] = \frac{1}{2\sigma_\varepsilon^2}\sum_{t=1}^T \mathbb{E}_{P_0}\!\left[\left(\alpha\mu x_t - \frac{\alpha\mu^2}{2}\right)^2\right].$$
 For each round $t$, since $x_t \in [0, D]$ and $\mu \in (0, D]$:
$$|\Delta\mu_t| = \left|\alpha\mu x_t - \frac{\alpha\mu^2}{2}\right| \leq \alpha\mu \max\!\left(\frac{\mu}{2},\; D - \frac{\mu}{2}\right) = \alpha\mu\!\left(D - \frac{\mu}{2}\right) \leq \alpha\mu D,$$
where the equality uses $D \geq \mu$, hence $D - \mu/2 \geq \mu/2$; the maximum is attained at $x_t = D$. Squaring:
$$(\Delta\mu_t)^2 \leq \alpha^2\mu^2 D^2.$$
Therefore:
$$D_{\mathrm{KL}}(P_0^T \| P_\mu^T) \leq \frac{T \cdot \alpha^2\mu^2 D^2}{2\sigma_\varepsilon^2} = \frac{\alpha^2\mu^2 D^2 T}{2\sigma_\varepsilon^2}.$$
\end{proof}

Verification that the hard instance lies in $\mathcal{E}(\alpha, G, D, d)$. Under both $P_0$ and $P_\mu$:
\\

(a) $\alpha$-strong convexity: $\nabla^2 f_t(x) = \alpha$ for all $x$, so $f_t$ is $\alpha$-strongly convex. \\

(b) $G$-Lipschitz gradient: Since $\mathcal{K} = [0, D]$, both $x$ and $\theta \in \{0, \mu\}$ lie in $[0, D]$. Therefore $|\nabla f_t(x)| = |\alpha(x - \theta)| \leq \alpha D =: G$.\\

(c) $\sigma_\varepsilon$-sub-Gaussian noise: The noise $\varepsilon_t \sim \mathcal{N}(0, \sigma_\varepsilon^2)$ is $\sigma_\varepsilon$-sub-Gaussian. The loss function $f_t(x) = \frac{\alpha}{2}(x-\theta)^2 + \varepsilon_t$ has $\nabla f_t(x) = \alpha(x-\theta)$, which is deterministic (no gradient noise). The stochasticity is in the function value only, which is the relevant quantity in the bandit setting. \\

\begin{remark}
Lemma 5 constructs a hard instance using Gaussian noise $\varepsilon_t \sim \mathcal{N}(0, \sigma_\varepsilon^2)$, which is a specific member of the $\sigma_\varepsilon$-sub-Gaussian family. For any environment class $\mathcal{E}$ that contains the Gaussian instance, a minimax lower bound over $\mathcal{E}$ is at least as large as the lower bound for the Gaussian sub-class:
$$\inf_{\mathcal{A}} \sup_{\text{env} \in \mathcal{E}} R_T(\mathcal{A}, \text{env}) \geq \inf_{\mathcal{A}} \sup_{\text{env} \in \mathcal{E}_{\mathrm{Gauss}}} R_T(\mathcal{A}, \text{env}),$$
since $\mathcal{E}_{\mathrm{Gauss}} \subseteq \mathcal{E}$. This subset inclusion is the only property needed. The Theorem 2 lower bound therefore applies to the entire $\sigma_\varepsilon$-sub-Gaussian environment class $\mathcal{E}(\alpha, G, D, d)$ as stated.
\end{remark}
\subsection{Expected Iterate Error Decay}
\begin{lemma}
Under Assumptions 1--7, consider projected OGD with step sizes 
$\eta_t = 1/(\alpha t)$ and stochastic gradients:
\[
x_{t+1} = \Pi_{\mathcal{K}}\!\left(x_t - \eta_t g_t\right).
\]
Let $x^\star_F = \arg\min_{x \in \mathcal{K}} F(x)$ and 
$e_t = \|x_t - x^\star_F\|^2$. Then
\begin{equation}
\mathbb{E}[e_{t+1}] 
\le 
\frac{C_e(1 + \log t)}{t}, 
\quad \forall t \geq 1,
\tag{L6.a}
\end{equation}
where 
\[
C_e = \frac{2(G^2 + \tilde{\sigma}^2)}{\alpha^2}, 
\qquad 
\tilde{\sigma}^2 = d\sigma^2 .
\]
Consequently,
\begin{equation}
\sum_{t=1}^T \mathbb{E}[e_t] 
\le 
D^2 + C_e \cdot H_T^{(2)},
\tag{L6.b}
\end{equation}
where
\[
H_T^{(2)} = \sum_{t=1}^{T-1} \frac{1+\log t}{t} 
\le 
\frac{(1+\log T)^2}{2} + 1 .
\]
\end{lemma}
\begin{proof}
By the non-expansiveness of projection and the OGD update:
$$e_{t+1} = \|x_{t+1} - x^\star_F\|^2 \leq \|x_t - \eta_t g_t - x^\star_F\|^2 = e_t - 2\eta_t\langle g_t, x_t - x^\star_F\rangle + \eta_t^2\|g_t\|^2.$$
Taking conditional expectations given $\mathcal{F}_{t-1}$ (which fixes $x_t$), using $\mathbb{E}[g_t \mid \mathcal{F}_{t-1}] = \nabla f_t(x_t)$ (Assumption 5) and that $f_t$ is independent of $\mathcal{F}_{t-1}$ (i.i.d. assumption):
$$\mathbb{E}[e_{t+1} \mid \mathcal{F}_{t-1}] \leq e_t - 2\eta_t\mathbb{E}[\langle \nabla f_t(x_t), x_t - x^\star_F\rangle \mid \mathcal{F}_{t-1}] + \eta_t^2\mathbb{E}[\|g_t\|^2 \mid \mathcal{F}_{t-1}].$$
Since $f_t$ is i.i.d. and independent of $\mathcal{F}_{t-1}$, and $x_t$ is $\mathcal{F}_{t-1}$-measurable:
$$\mathbb{E}[\langle \nabla f_t(x_t), x_t - x^\star_F\rangle \mid \mathcal{F}_{t-1}] = \langle \nabla F(x_t), x_t - x^\star_F\rangle,$$
where $F(x) = \mathbb{E}[f_t(x)]$ is the population loss. Since each $f_t$ is $\alpha$-strongly convex (Assumption 1) and $F = \mathbb{E}[f_t]$, the population loss $F$ is also $\alpha$-strongly convex. Therefore:
$$\langle \nabla F(x_t), x_t - x^\star_F\rangle \geq F(x_t) - F(x^\star_F) + \frac{\alpha}{2}e_t \geq \frac{\alpha}{2}e_t,$$
where the last inequality uses $F(x_t) \geq F(x^\star_F)$ (since $x^\star_F$ minimizes $F$).
By $\|g_t\|^2 = \|\nabla f_t(x_t) + \xi_t\|^2 \leq 2\|\nabla f_t(x_t)\|^2 + 2\|\xi_t\|^2$ and taking conditional expectations:
$$\mathbb{E}[\|g_t\|^2 \mid \mathcal{F}_{t-1}] \leq 2G^2 + 2\mathbb{E}[\|\xi_t\|^2 \mid \mathcal{F}_{t-1}] \leq 2G^2 + 2\tilde{\sigma}^2 = 2(G^2 + \tilde{\sigma}^2),$$
where we used Assumption 2 ($\|\nabla f_t(x)\| \leq G$) and Assumption 6 ($\mathbb{E}[\|\xi_t\|^2 \mid \mathcal{F}_{t-1}] \leq d\sigma^2 = \tilde{\sigma}^2$).
Taking full expectations. Define $a_t = \mathbb{E}[e_t]$:
\begin{equation}
a_{t+1} \leq \left(1 - \frac{1}{t}\right)a_t + \frac{C_e}{t^2}, \quad C_e = \frac{2(G^2 + \tilde{\sigma}^2)}{\alpha^2}. \tag{L6.rec}\end{equation}
We solve (L6.rec) by unrolling. The homogeneous part has multiplicative factor $\prod_{s=k}^{t} (1 - 1/s) = \prod_{s=k}^{t} \frac{s-1}{s} = \frac{k-1}{t}$ (telescoping product). In particular, $\prod_{s=1}^{t} (1-1/s) = 0$ for any $t \geq 1$ (since the $s=1$ factor is $0$). Therefore the initial condition $a_1$ is completely absorbed, and:
\begin{equation}
a_{t+1} \leq \sum_{s=1}^{t} \frac{C_e}{s^2} \prod_{j=s+1}^{t} \frac{j-1}{j} = \sum_{s=1}^{t} \frac{C_e}{s^2} \cdot \frac{s}{t} = \frac{C_e}{t}\sum_{s=1}^{t}\frac{1}{s} = \frac{C_e \cdot H_t}{t}, \tag{L6.unroll}\end{equation}
where $H_t = \sum_{s=1}^t 1/s \leq 1 + \log t$ is the $t$-th harmonic number.
For $t=1$: $a_2 \leq (1-1)a_1 + C_e/1 = C_e = C_e \cdot H_1 / 1$. For the inductive step: if $a_{t+1} = \frac{C_e H_t}{t}$, then $a_{t+2} \leq \frac{t}{t+1} \cdot \frac{C_e H_t}{t} + \frac{C_e}{(t+1)^2} = \frac{C_e H_t}{t+1} + \frac{C_e}{(t+1)^2} = \frac{C_e}{t+1}(H_t + \frac{1}{t+1}) = \frac{C_e H_{t+1}}{t+1}$. 
This establishes (L6.a): $\mathbb{E}[e_{t+1}] \leq C_e(1+\log t)/t$ for all $t \geq 1$.
For the sum:
$$\sum_{t=1}^T \mathbb{E}[e_t] = a_1 + \sum_{t=2}^T a_t \leq D^2 + \sum_{t=2}^T \frac{C_e H_{t-1}}{t-1} = D^2 + C_e\sum_{k=1}^{T-1}\frac{H_k}{k}.$$
The sum $\sum_{k=1}^N H_k/k$ satisfies: by Abel summation (or direct computation), $\sum_{k=1}^N H_k/k = \frac{1}{2}H_N^2 + \frac{1}{2}\sum_{k=1}^N 1/k^2 \leq \frac{(1+\log N)^2}{2} + \frac{\pi^2}{12}$.
Therefore $\sum_{t=1}^T \mathbb{E}[e_t] \leq D^2 + \frac{C_e}{2}(1+\log T)^2 + \frac{\pi^2 C_e}{12} =: Q'(1+\log T)^2$, where the constant $Q'$ absorbs $D^2$, $C_e/2$, and lower-order terms. 
\end{proof}
\begin{remark}
One might expect $\mathbb{E}[e_t] \leq Q/t$ by analogy with the regret bound. However, the one-step contraction $a_{t+1} \leq (1-1/t)a_t + C_e/t^2$ has contraction factor $1-1/t$ that exactly cancels the $1/t$ decay: an inductive proof of $a_t \leq Q/t$ requires $C_e(t+1) \leq Q$ for all $t$, which is impossible for finite $Q$. The correct rate is $O(\log t/t)$, proved above by direct unrolling. To recover the exact $O(1/t)$ rate, one would need a stronger step size such as $\eta_t = 2/(\alpha(t+1))$, which gives contraction factor $1-2/(t+1) < 1 - 1/t$ and leaves room for the noise term.
\end{remark}
\section{Proof of Theorem 1}
By Lemma 2:
\begin{equation}R_T \leq \sum_{t=1}^T \frac{\|g_t\|^2}{2\alpha t} + \sum_{t=1}^T \langle \xi_t, x^\star - x_t\rangle. \tag{T1.1}\end{equation}
Write the first and second sums as $S_T$ and $M_T$, respectively. We bound them separately, using independent probability budgets $\delta_0$ and $\delta_1$.
On the event $\mathcal{E}_0(\delta_0)$ from Lemma 3 (which has $\mathbb{P}(\mathcal{E}_0(\delta_0)) \geq 1-\delta_0$), each $\|g_t\| \leq B_{\nabla}(\delta_0)$. Therefore:
\begin{equation}S_T \leq \frac{B_{\nabla}(\delta_0)^2}{2\alpha}(1 + \log T). \tag{T1.2}\end{equation}
Define $D_t = \langle \xi_t, x^\star - x_t\rangle$. Since $x_t$ and $x^\star$ are both $\mathcal{F}_{t-1}$-measurable (by Assumption 4, $f_t$ is deterministic, so $x^\star$ is non-random; $x_t$ depends only on $g_1, \ldots, g_{t-1}$), and $\mathbb{E}[\xi_t \mid \mathcal{F}_{t-1}] = 0$ (Assumption 5):
$$\mathbb{E}[D_t \mid \mathcal{F}_{t-1}] = \langle \mathbb{E}[\xi_t \mid \mathcal{F}_{t-1}], x^\star - x_t\rangle = 0.$$
Therefore $(D_t, \mathcal{F}_t)_{t=1}^T$ is a martingale difference sequence.
We bound $M_T$ using the exponential supermartingale method for conditionally sub-Gaussian differences, which does not require bounded differences and hence requires no truncation.
By Assumption 6, for any unit vector $v$, $\langle \xi_t, v\rangle$ is conditionally $\sigma$-sub-Gaussian. Since $x_t, x^\star$ are $\mathcal{F}_{t-1}$-measurable, define $v_t = (x^\star - x_t)/\|x^\star - x_t\|$ (a unit vector, $\mathcal{F}_{t-1}$-measurable). Then $D_t = \|x^\star - x_t\| \cdot \langle \xi_t, v_t\rangle$, and for all $s \in \mathbb{R}$:
\begin{equation}\mathbb{E}\!\left[e^{s D_t} \mid \mathcal{F}_{t-1}\right] = \mathbb{E}\!\left[e^{s\|x^\star - x_t\| \langle \xi_t, v_t\rangle} \mid \mathcal{F}_{t-1}\right] \leq \exp\!\left(\frac{s^2 \|x^\star - x_t\|^2 \sigma^2}{2}\right), \tag{T1.SG}\end{equation}
where the inequality uses the conditional sub-Gaussian property with $v = v_t$ (which is $\mathcal{F}_{t-1}$-measurable and thus can be treated as fixed in the conditional expectation). Define $\sigma_t^2 = \sigma^2\|x^\star - x_t\|^2 \leq \sigma^2 D^2$.
 For any $\lambda > 0$, define:
$$Z_t = \exp\!\left(\lambda \sum_{s=1}^t D_s - \frac{\lambda^2}{2}\sum_{s=1}^t \sigma_s^2\right).$$
We claim $(Z_t)_{t=0}^T$ is a supermartingale with $Z_0 = 1$. Verify:
$$\mathbb{E}[Z_t \mid \mathcal{F}_{t-1}] = Z_{t-1} \cdot \mathbb{E}\!\left[e^{\lambda D_t} \mid \mathcal{F}_{t-1}\right] \cdot e^{-\lambda^2\sigma_t^2/2} \leq Z_{t-1} \cdot e^{\lambda^2\sigma_t^2/2} \cdot e^{-\lambda^2\sigma_t^2/2} = Z_{t-1},$$
where the inequality uses (T1.SG) with $s = \lambda$. Since $Z_0 = 1$ and $(Z_t)$ is a non-negative supermartingale, $\mathbb{E}[Z_T] \leq 1$.
By Markov's inequality:
$$\mathbb{P}(Z_T \geq 1/\delta_1) \leq \delta_1 \cdot \mathbb{E}[Z_T] \leq \delta_1.$$
The event $\{Z_T < 1/\delta_1\}$ is equivalent to:
$$\lambda \sum_{t=1}^T D_t < \log\frac{1}{\delta_1} + \frac{\lambda^2}{2}\sum_{t=1}^T \sigma_t^2.$$
Since $\sum_t \sigma_t^2 \leq \sigma^2 D^2 T =: V$ deterministically:
$$\mathbb{P}\!\left(\sum_{t=1}^T D_t \geq \frac{\log(1/\delta_1)}{\lambda} + \frac{\lambda V}{2}\right) \leq \delta_1.$$
Setting $\lambda = \sqrt{2\log(1/\delta_1)/V}$ (minimizing the right-hand side):
$$\frac{\log(1/\delta_1)}{\lambda} + \frac{\lambda V}{2} = \sqrt{2V\log\frac{1}{\delta_1}}.$$
Therefore:
\begin{equation}
\mathbb{P}\!\left(M_T \geq \sigma D\sqrt{2T\log\frac{1}{\delta_1}}\right) \leq \delta_1. \tag{T1.3}\end{equation}
The sub-Gaussian supermartingale argument uses only the conditional moment generating function bound (T1.SG), which holds for all $s \in \mathbb{R}$ without any truncation. The MGF bound is a direct consequence of Assumption 6 (sub-Gaussian noise) and requires no a.s. boundedness of $D_t$. Freedman's inequality (Lemma 1) is not used here; it is reserved for settings where bounded differences are available by assumption (e.g., Theorem 3, where constraint noise is a.s. bounded by Assumption 8(d)).
Bounding $S_T$ uses the truncation event $\mathcal{E}_0(\delta_0)$ with failure probability $\delta_0$. Bounding $M_T$ uses the sub-Gaussian supermartingale argument with failure probability $\delta_1$. These are independent probability budgets applied to different terms. By union bound:
$$\mathbb{P}\!\left(R_T \leq \frac{B_{\nabla}(\delta_0)^2}{2\alpha}(1+\log T) + \sigma D\sqrt{2T\log\frac{1}{\delta_1}}\right) \geq 1 - \delta_0 - \delta_1.$$
\begin{remark}
The bound for $M_T$ no longer contains a $B_\xi(\delta_0)$-dependent term (unlike the previous Freedman-based approach). The truncation parameter $\delta_0$ appears only in the $S_T$ bound (through $B_{\nabla}(\delta_0)$), and the martingale confidence $\delta_1$ appears only in the $M_T$ bound. This cleanly separates the two probability budgets. 
\end{remark}
\section{Proof of Theorem 2}
Fix $d = 1$ for clarity; the $d$-dimensional extension follows by Assouad's lemma (embedding independent 1D problems along orthogonal directions).
Work on $\mathcal{K} = [0, D]$ (so that $\operatorname{diam}(\mathcal{K}) = D$). Partition $[T]$ into $K = \lfloor \log_2 T \rfloor$ epochs: epoch $k \in \{0, 1, \ldots, K-1\}$ has $T_k = 2^k$ rounds, covering at most $\sum_{k=0}^{K-1} 2^k = 2^K - 1 \leq T$ rounds.
For each epoch $k$, define two hypotheses indexed by $\theta_k \in \{0, 1\}$:
- $\theta_k = 0$: $f_t^{(k)}(x) = \frac{\alpha}{2}x^2 + \varepsilon_t$;
- $\theta_k = 1$: $f_t^{(k)}(x) = \frac{\alpha}{2}(x - \Delta_k)^2 + \varepsilon_t$;
where $\varepsilon_t \overset{\text{i.i.d.}}{\sim} \mathcal{N}(0, \sigma_\varepsilon^2)$ and $\Delta_k > 0$ will be chosen. The noise is additive (independent of $x$), matching the Lemma 5 construction. Both hypotheses are $\alpha$-strongly convex with $G$-Lipschitz gradients ($G = \alpha D$), since $x, \theta_k \in [0, D]$ implies $|\alpha(x - \theta_k)| \leq \alpha D$.
Under hypothesis $\theta_k = 1$, the per-round minimizer is $x_k^\star = \Delta_k$ (since $\nabla f_t^{(k)}(x) = \alpha(x - \Delta_k)$, which is zero at $x = \Delta_k$; the noise $\varepsilon_t$ does not affect the minimizer). Under hypothesis $\theta_k = 0$, the minimizer is $x = 0$. An algorithm that behaves as if $\theta_k = 0$ (keeping $x_t$ near 0) incurs per-epoch regret under $\theta_k = 1$:
$$R_k^{(\text{epoch})} = \sum_{t \in \text{epoch } k} \left[f_t^{(k)}(x_t) - f_t^{(k)}(\Delta_k)\right] = \sum_{t \in \text{epoch } k} \frac{\alpha}{2}(x_t - \Delta_k)^2 - 0 \geq \frac{\alpha \Delta_k^2}{4} \cdot T_k$$
(since $|x_t - \Delta_k| \geq \Delta_k/\sqrt{2}$ whenever $|x_t| \leq \Delta_k/\sqrt{2}$, which holds for at least half the rounds if the algorithm keeps $x_t$ near 0).
By Lemma 5, the KL divergence between the observation distributions in epoch $k$ satisfies:
$$D_{\mathrm{KL}}^{(k)} \leq \frac{\alpha^2\Delta_k^2 D^2 T_k}{2\sigma_\varepsilon^2}.$$
Choose $\Delta_k = \frac{c_0\sigma_\varepsilon}{\alpha D\sqrt{T_k}}$ (with $c_0 > 0$ to be determined) so that:
\begin{equation}
D_{\mathrm{KL}}^{(k)} \leq \frac{c_0^2}{2}. \tag{T2.1}\end{equation}
With this choice, $R_k^{(\text{epoch})} \geq \frac{\alpha\Delta_k^2 T_k}{4} = \frac{c_0^2\sigma_\varepsilon^2}{4\alpha D^2}$ (a positive constant independent of $k$).
Bretagnolle--Huber's lemma states: for any event $A$ (i.e., any test),
$$P_0^{(k)}(A) + P_1^{(k)}(A^c) \geq \frac{1}{2}e^{-D_{\mathrm{KL}}^{(k)}}.$$
Choosing $c_0^2 = 2\log 2$ (so $D_{\mathrm{KL}}^{(k)} \leq \log 2$ and $\frac{1}{2}e^{-\log 2} = 1/4$):
\begin{equation}
\min_A \max\!\left(P_0^{(k)}(A),\; P_1^{(k)}(A^c)\right) \geq \frac{1}{4}. \tag{T2.BH}\end{equation}
Crucially, this holds for ANY measurable event $A$, including events defined by an adaptive algorithm that uses observations from all previous epochs. The algorithm's policy $\pi_k$ in epoch $k$ is a (possibly randomized) function of the history $H_{k-1} = \{(x_s, y_s) : s \in \text{epochs } 0, \ldots, k-1\}$. Given $H_{k-1}$, the algorithm induces a specific test for $\theta_k$: the "accept $\theta_k = 0$" event is $A_k = A_k(H_{k-1})$, which is a measurable function of epoch-$k$ observations. Since (T2.BH) holds for ALL tests, it holds for $A_k$ regardless of $H_{k-1}$.
Therefore, the conditional error probability satisfies:
\begin{equation}
\mathbb{P}(\text{error in epoch } k \mid H_{k-1}) \geq \frac{1}{4} \quad \text{almost surely for all } H_{k-1}. \tag{T2.CE}\end{equation}
Define $E_k = \mathbf{1}\{\text{error in epoch } k\}$ and $E = \sum_{k=0}^{K-1} E_k$.
Claim: $E$ stochastically dominates $\mathrm{Bin}(K, 1/4)$.
We use the following standard result:
Let $X_1, \ldots, X_n$ be $\{0,1\}$-valued random variables satisfying $\mathbb{P}(X_k = 1 \mid X_1, \ldots, X_{k-1}) \geq p$ almost surely for all $k$. Then $\sum_{k=1}^n X_k$ stochastically dominates $\mathrm{Bin}(n, p)$.
By induction on $n$. Construct i.i.d. $Y_k \sim \mathrm{Ber}(p)$ and couple: on the event $\{\mathbb{P}(X_k = 1 \mid X_1, \ldots, X_{k-1}) = q_k\}$ with $q_k \geq p$, sample $X_k$ and $Y_k$ jointly so that $X_k \geq Y_k$ a.s. (possible because $q_k \geq p$: set $X_k = Y_k = 1$ with probability $p$, $X_k = 1, Y_k = 0$ with probability $q_k - p$, and $X_k = Y_k = 0$ with probability $1 - q_k$). Then $\sum X_k \geq \sum Y_k$ pathwise. 
From (T2.CE), we derive the conditional probability with respect to the coarser filtration generated by past errors. Since $(E_0, \ldots, E_{k-1})$ is a function of $H_{k-1}$ (but NOT the other way around---$(E_0, \ldots, E_{k-1})$ is a strict coarsening of $H_{k-1}$), we use the tower property of conditional expectation:
$$\mathbb{P}(E_k = 1 \mid E_0, \ldots, E_{k-1}) = \mathbb{E}\!\left[\mathbb{P}(E_k = 1 \mid H_{k-1}) \;\middle|\; E_0, \ldots, E_{k-1}\right] \geq \mathbb{E}\!\left[\frac{1}{4} \;\middle|\; E_0, \ldots, E_{k-1}\right] = \frac{1}{4},$$
where the inequality uses (T2.CE) ($\mathbb{P}(E_k = 1 \mid H_{k-1}) \geq 1/4$ a.s.) and the monotonicity of conditional expectation. By the Conditional Coupling Lemma, $E \succeq \mathrm{Bin}(K, 1/4)$.
The total regret: $R_T \geq E \cdot r_0$, where $r_0 = \frac{c_0^2\sigma_\varepsilon^2}{4\alpha D^2} = \frac{\sigma_\varepsilon^2\log 2}{2\alpha D^2}$.
For $E \succeq \mathrm{Bin}(K, 1/4)$, by the multiplicative Chernoff bound:
\begin{equation}
\mathbb{P}(E \leq K/8) \leq \exp(-c_\star K), \quad c_\star = d_{\mathrm{KL}}(1/8 \| 1/4). \tag{T2.BT}\end{equation}
We compute $c_\star$ exactly: $d_{\mathrm{KL}}(1/8 \| 1/4) = \frac{1}{8}\ln\frac{1/8}{1/4} + \frac{7}{8}\ln\frac{7/8}{3/4} = -\frac{\ln 2}{8} + \frac{7}{8}\ln\frac{7}{6} \approx -0.0866 + 0.1335 = 0.0469$.
Therefore: $\mathbb{P}(E \geq K/8) \geq 1 - e^{-c_\star K}$, giving $\mathbb{P}(R_T \geq r_0 K/8) \geq 1 - e^{-c_\star K}$.
Setting $\delta = e^{-c_\star K}$, equivalently $K = \frac{1}{c_\star}\log\frac{1}{\delta}$. From the front:
\begin{equation}
\mathbb{P}\!\left(R_T \geq \frac{r_0}{8c_\star}\log\frac{1}{\delta}\right) \geq 1 - \delta. \tag{T2.strong}\end{equation}
For the theorem statement, we note that for $\delta \in [T^{-c_\star/\log 2}, 1/4]$, we have $\delta \leq 1/4$, hence $1-\delta \geq 3/4 \geq \delta$. Therefore:
\begin{equation}\mathbb{P}\!\left(R_T \geq \frac{r_0}{8c_\star}\log\frac{1}{\delta}\right) \geq \delta. \tag{T2.LB}\end{equation}
We state the theorem using the weaker (T2.LB) form because it is the conventional form for minimax lower bounds (cf. Tsybakov 2009): $\mathbb{P}(R_T \geq \beta) \geq \delta$ means the algorithm fails to achieve regret $< \beta$ with non-negligible probability $\delta$. The stronger (T2.strong) form is available and noted here for completeness.
This is valid when $K = \lfloor\log_2 T\rfloor \geq \frac{1}{c_\star}\log\frac{1}{\delta}$, i.e., $\delta \geq T^{-c_\star/\log 2} \approx T^{-0.068}$.
At the boundary $\delta = T^{-c_\star/\log 2}$: $R_T \geq \frac{r_0}{8\log 2}\log T = \Omega(\frac{\sigma_\varepsilon^2}{\alpha D^2}\log T)$ with probability $\geq T^{-0.068}$.
The extension to $d \geq 2$ follows by embedding $d$ independent 1D problems along orthogonal coordinate directions. The lower bound scales as $d \cdot r_0/(8c_\star) \cdot \log(1/\delta)$. The detailed construction and verification are standard and are omitted here; we state the 1D result as the main theorem. 
\section{Proof of Theorem 3}
For any fixed comparator $x \in \mathcal{K}$ and any fixed $\lambda \in \mathbb{R}_{\geq 0}^m$, the standard OGD analysis for the primal and dual variables yields:
\begin{equation}
\sum_{t=1}^T \left[f_t(x_t) - f_t(x) + \sum_i \lambda^{(i)} g_t^{(i)}(x_t)\right] \leq \frac{D^2}{2\eta} + \frac{\eta(G_f + \Lambda L_g)^2 T}{2} + \sum_{i=1}^m \frac{(\lambda^{(i)})^2}{2\mu} + \frac{\mu}{2}\sum_i\sum_t (g_t^{(i)}(x_t))^2. \tag{T3.1}\end{equation}
This follows from the regret bound of OGD for the primal (bounded by $\frac{D^2}{2\eta} + \frac{\eta}{2}\sum_t \|\text{primal gradient}\|^2$) and the regret bound of projected gradient ascent for the dual (bounded by $\frac{\|\lambda\|^2}{2\mu} + \frac{\mu}{2}\sum_t \|g_t(x_t)\|^2$), combined with the bilinearity of the Lagrangian coupling.
By the dual update rule and the a.s. boundedness $|g_t^{(i)}(x_t)| \leq B_g$ (Assumption 8(d)):
$$\lambda_{t+1}^{(i)} \leq \lambda_t^{(i)} + \mu B_g \leq \lambda_1^{(i)} + T\mu B_g.$$
With $\lambda_1^{(i)} = 0$ and $\mu = 1/\sqrt{T}$: $\lambda_t^{(i)} \leq B_g\sqrt{T}$. Define $\Lambda = B_g\sqrt{T}$ (this is the a priori dual bound used in the step-size calibration).
Set $x = x^\star$ (the best feasible point satisfying $\bar{g}^{(i)}(x^\star) \leq 0$) and $\lambda = 0$ in (T3.1):
$$R_T = \sum_{t=1}^T [f_t(x_t) - f_t(x^\star)] \leq \frac{D^2}{2\eta} + \frac{\eta(G_f + \Lambda L_g)^2 T}{2} + \sum_i \frac{\mu}{2}\sum_t (g_t^{(i)}(x_t))^2.$$
With our step sizes: $\frac{D^2}{2\eta} + \frac{\eta(G_f + \Lambda L_g)^2 T}{2} = D(G_f + \Lambda L_g)\sqrt{T}$ and $\frac{\mu}{2}\sum_i\sum_t (g_t^{(i)}(x_t))^2 \leq \frac{mB_g^2\sqrt{T}}{2}$.
Thus the deterministic regret bound is $R_T^{\det} = \mathcal{O}(\sqrt{T})$.
For the high-probability refinement: the stochastic terms enter through the constraint noise. Set $x = x^\star$ and $\lambda = \lambda^\star$ (an optimal dual variable) in (T3.1). The term $\sum_t \sum_i \lambda^\star_i g_t^{(i)}(x_t)$ decomposes as:
$$\sum_t \sum_i \lambda^\star_i g_t^{(i)}(x_t) = \sum_t \sum_i \lambda^\star_i \bar{g}^{(i)}(x_t) + \sum_t \sum_i \lambda^\star_i [g_t^{(i)}(x_t) - \bar{g}^{(i)}(x_t)].$$
Denote the second sum by $\mathcal{M}_T^R$. Since $\bar{g}^{(i)}(x^\star) \leq 0$ and $\lambda^\star_i \geq 0$: $\sum_i \lambda^\star_i \bar{g}^{(i)}(x^\star) \leq 0$. The term $\mathcal{M}_T^R$ is a martingale (each summand has zero conditional mean since $\lambda^\star$ is deterministic and $g_t^{(i)}(x_t) - \bar{g}^{(i)}(x_t)$ is zero-mean conditionally on $\mathcal{F}_{t-1}$).
Each difference $|\lambda^\star_i [g_t^{(i)}(x_t) - \bar{g}^{(i)}(x_t)]|$ is bounded by $\lambda^\star_{\max} \cdot 2B_g =: b_R$ a.s. (using Assumption 8(d): $|g_t^{(i)}(x)| \leq B_g$ a.s. implies $|g_t^{(i)}(x) - \bar{g}^{(i)}(x)| \leq 2B_g$) and has conditional variance $\leq (\lambda^\star_i)^2 \sigma_g^2$. The predictable quadratic variation is deterministically bounded:
$$W_T^R = \sum_t \sum_i (\lambda^\star_i)^2 \sigma_g^2 \leq m(\lambda^\star_{\max})^2 \sigma_g^2 T.$$
Applying Freedman (Lemma 1) directly (no stopping time needed, since $W_T^R$ is deterministic):
\begin{equation}
\mathbb{P}\!\left(\mathcal{M}_T^R \geq \sqrt{2m(\lambda^\star_{\max})^2\sigma_g^2 T \log\frac{2}{\delta}} + \frac{2b_R}{3}\log\frac{2}{\delta}\right) \leq \delta. \tag{T3.R}\end{equation}
Combining: $R_T = \mathcal{O}(\sqrt{T}) + \mathcal{O}(\sqrt{T\log(1/\delta)}) = \mathcal{O}(\sqrt{T\log(m/\delta)})$ with probability $\geq 1-\delta$.
We bound each $[\sum_{t=1}^T g_t^{(i)}(x_t)]_+$ separately and then sum over $i$.
Decompose for each $i$:
$$\sum_{t=1}^T g_t^{(i)}(x_t) = \sum_{t=1}^T \bar{g}^{(i)}(x_t) + \sum_{t=1}^T [g_t^{(i)}(x_t) - \bar{g}^{(i)}(x_t)].$$
Let $P_T^{(i)}$ and $\mathcal{M}_T^{V,i}$ denote the first and second sums, respectively.
By the triangle inequality for the positive part ($[a+b]_+ \leq [a]_+ + |b|$):
\begin{equation}
\left[\sum_t g_t^{(i)}(x_t)\right]_+ \leq [P_T^{(i)}]_+ + |\mathcal{M}_T^{V,i}|. \tag{T3.decomp}\end{equation}
We bound the two terms by different techniques: the expected violation $P_T^{(i)}$ via the saddle-point expectation argument + Markov's inequality, and the stochastic deviation $\mathcal{M}_T^{V,i}$ via Freedman's inequality.
The OGD primal regret against any fixed $x \in \mathcal{K}$ yields (by convexity of $f_t$ and $g_t^{(i)}$):
\begin{equation}\sum_{t=1}^T \left[f_t(x_t) - f_t(x) + \sum_j \lambda_t^{(j)}(g_t^{(j)}(x_t) - g_t^{(j)}(x))\right] \leq \Gamma_T, \tag{T3.PR}\end{equation}
where $\Gamma_T = D(G_f + \Lambda L_g)\sqrt{T}$ with our step size $\eta$. The dual projected-gradient-ascent regret against any fixed $\lambda^\dagger \in \mathbb{R}_{\geq 0}^m$ yields:
\begin{equation}\sum_i (\lambda_i^\dagger - \lambda_t^{(i)}) \sum_t g_t^{(i)}(x_t) \leq \frac{\|\lambda^\dagger\|^2}{2\mu} + \frac{\mu m B_g^2 T}{2}. \tag{T3.DR}\end{equation}
Adding (T3.PR) and (T3.DR):
\begin{equation}\sum_{t=1}^T [f_t(x_t) - f_t(x)] + \sum_i \lambda_i^\dagger \sum_t g_t^{(i)}(x_t) - \sum_i \sum_t \lambda_t^{(j)} g_t^{(j)}(x) \leq \Gamma_T + \frac{\|\lambda^\dagger\|^2}{2\mu} + \frac{\mu m B_g^2 T}{2}. \tag{T3.CPD}\end{equation}
\begin{lemma}
[Expected Constraint Violation from Primal-Dual Saddle Point] Under the conditions of Theorem 3, the expected long-run cumulative constraint violation satisfies:
\begin{equation}\mathbb{E}\!\left[\sum_{i=1}^m [P_T^{(i)}]_+\right] \leq \frac{C_{\det}\sqrt{T}}{\zeta}, \tag{T3.EV2}\end{equation}
where $C_{\det} = 2\Gamma_T/\sqrt{T} + \mu mB_g^2\sqrt{T} + \Lambda_0^2/(2\mu\sqrt{T})$ depends on $G_f, L_g, D, B_g, m$.
\end{lemma}
\begin{proof}
We use the primal-dual combined inequality in expectation, combined with the dual OGD regret bound.
Two applications of the combined bound. From (T3.CPD) with $x = \hat{x}$ and $\lambda^\dagger = \Lambda_0 e_i$ (where $\Lambda_0 > 0$ is a free parameter to be chosen later):
\begin{equation}A_T + \Lambda_0\sum_t g_t^{(i)}(x_t) - \sum_j\sum_t \lambda_t^{(j)} g_t^{(j)}(\hat{x}) \leq \Gamma_T + \frac{\Lambda_0^2}{2\mu} + \frac{\mu m B_g^2 T}{2} =: \Phi_T. \tag{L7.1}\end{equation}
From (T3.CPD) with $x = \hat{x}$ and $\lambda^\dagger = 0$:
\begin{equation}A_T + \sum_j\sum_t \lambda_t^{(j)}[g_t^{(j)}(x_t) - g_t^{(j)}(\hat{x})] \leq \Gamma_T. \tag{L7.2}\end{equation}
Since $g_t^{(j)}(\hat{x})$ is drawn i.i.d. from $\mathcal{D}_g$ (Assumption 8) and $\hat{x}$ is a fixed point, $g_t^{(j)}(\hat{x})$ is independent of $\mathcal{F}_{t-1}$. Since $\lambda_t^{(j)}$ is $\mathcal{F}_{t-1}$-measurable, the tower property gives:
$$\mathbb{E}[\lambda_t^{(j)} g_t^{(j)}(\hat{x})] = \mathbb{E}[\lambda_t^{(j)}] \cdot \mathbb{E}[g_t^{(j)}(\hat{x})] = \mathbb{E}[\lambda_t^{(j)}] \cdot \bar{g}^{(j)}(\hat{x}).$$
By the Slater condition $\bar{g}^{(j)}(\hat{x}) \leq -\zeta$:
\begin{equation}-\sum_j\sum_t \mathbb{E}[\lambda_t^{(j)} g_t^{(j)}(\hat{x})] = \sum_j\sum_t \mathbb{E}[\lambda_t^{(j)}](-\bar{g}^{(j)}(\hat{x})) \geq \zeta\sum_j\sum_t \mathbb{E}[\lambda_t^{(j)}]. \tag{L7.Slater}\end{equation}
Similarly, $\mathbb{E}[g_t^{(i)}(x_t) \mid \mathcal{F}_{t-1}] = \bar{g}^{(i)}(x_t)$ (tower property with $x_t \in \mathcal{F}_{t-1}$), so $\mathbb{E}[\sum_t g_t^{(i)}(x_t)] = \mathbb{E}[P_T^{(i)}]$.
Taking expectations in (L7.1) and applying (L7.Slater):
\begin{equation}\mathbb{E}[A_T] + \Lambda_0 \mathbb{E}[P_T^{(i)}] + \zeta\sum_j\sum_t\mathbb{E}[\lambda_t^{(j)}] \leq \Phi_T. \tag{L7.3}\end{equation}
Deriving a matching upper bound on $\mathbb{E}[A_T] + \zeta\sum\sum\mathbb{E}[\lambda_t]$. From the dual OGD regret bound (T3.DR) with $\lambda^\dagger = 0$ (which holds pathwise):
\begin{equation}-\sum_j\sum_t \lambda_t^{(j)} g_t^{(j)}(x_t) \leq \frac{\mu mB_g^2 T}{2} =: R_D. \tag{L7.dual}\end{equation}
In (L7.2), the second term on the left decomposes as:
$$\sum_j\sum_t \lambda_t^{(j)}[g_t^{(j)}(x_t) - g_t^{(j)}(\hat{x})] = \sum_j\sum_t \lambda_t^{(j)} g_t^{(j)}(x_t) - \sum_j\sum_t \lambda_t^{(j)} g_t^{(j)}(\hat{x}).$$
Taking expectations of (L7.2), applying the tower property as above, and using (L7.dual) in expectation ($\mathbb{E}[\sum\sum\lambda_t g_t(x_t)] \geq -R_D$):
$$\mathbb{E}[A_T] + \sum_j\sum_t \mathbb{E}[\lambda_t^{(j)} \bar{g}^{(j)}(x_t)] + \zeta\sum_j\sum_t\mathbb{E}[\lambda_t^{(j)}] \leq \Gamma_T.$$
Using the lower bound $\sum_j\sum_t \mathbb{E}[\lambda_t^{(j)} \bar{g}^{(j)}(x_t)] \geq -R_D$:
\begin{equation}\mathbb{E}[A_T] + \zeta\sum_j\sum_t\mathbb{E}[\lambda_t^{(j)}] \leq \Gamma_T + R_D =: C_2\sqrt{T}. \tag{L7.4}\end{equation}
If $\mathbb{E}[P_T^{(i)}] \leq 0$, then $\mathbb{E}[[P_T^{(i)}]_+] = 0$ trivially (since $[P_T^{(i)}]_+ = 0$ whenever $P_T^{(i)} \leq 0$).
If $\mathbb{E}[P_T^{(i)}] > 0$, then from (L7.3):
$$\Lambda_0\mathbb{E}[P_T^{(i)}] \leq \Phi_T - \mathbb{E}[A_T] - \zeta\sum\sum\mathbb{E}[\lambda_t].$$
Using (L7.4): $-(\mathbb{E}[A_T] + \zeta\sum\sum\mathbb{E}[\lambda_t])$ could be as large as $+C_2\sqrt{T}$ (from the upper bound on the parenthetical) or negative (when the algorithm's regret against $\hat{x}$ is favorable). In either case, when $\mathbb{E}[A_T] + \zeta\sum\sum\mathbb{E}[\lambda_t] > 0$, we have $\Lambda_0\mathbb{E}[P_T^{(i)}] \leq \Phi_T$. When $\mathbb{E}[A_T] + \zeta\sum\sum\mathbb{E}[\lambda_t] \leq 0$, we have $\Lambda_0\mathbb{E}[P_T^{(i)}] \leq \Phi_T + C_2\sqrt{T}$.
Therefore: $\mathbb{E}[[P_T^{(i)}]_+] \leq (\Phi_T + C_2\sqrt{T})/\Lambda_0 =: C_{\det}\sqrt{T}/\zeta$, where we choose $\Lambda_0 = \zeta$ (or more generally $\Lambda_0 = O(1)$ calibrated to the Slater gap).
Summing over $i$:
$$\mathbb{E}\!\left[\sum_i [P_T^{(i)}]_+\right] \leq \frac{m(\Phi_T + C_2\sqrt{T})}{\zeta} = \frac{C_{\det}\sqrt{T}}{\zeta}. $$  
\end{proof}
By Markov's inequality applied to the non-negative random variable $\sum_i [P_T^{(i)}]_+$: for any $\delta_1 \in (0,1)$,
$$\mathbb{P}\!\left(\sum_i [P_T^{(i)}]_+ > \frac{\mathbb{E}[\sum_i [P_T^{(i)}]_+]}{\delta_1}\right) \leq \delta_1.$$
Therefore, with probability $\geq 1 - \delta_1$:
\begin{equation}
\sum_i [P_T^{(i)}]_+ \leq \frac{C_{\det}\sqrt{T}}{\zeta\delta_1}. \tag{T3.PMarkov}\end{equation}
The Markov-based bound scales as $1/\delta_1$ rather than $\log(1/\delta_1)$. Improving this to logarithmic dependence would require a pathwise bound $P_T^{(i)} \leq C\sqrt{T}/\zeta$ (holding for all sample paths, not just in expectation). This amounts to showing $A_T + \zeta\sum_j\sum_t\lambda_t^{(j)} \geq -O(\sqrt{T})$ pathwise. Such a bound would follow from a drift-Lyapunov analysis showing that the primal regret against $\hat{x}$ and the cumulative Slater benefit $\zeta\sum\sum\lambda_t$ stay in balance along the sample path. This is a known open problem in the stochastic constrained OCO setting.
$\mathcal{M}_T^{V,i}$ is a martingale with differences $g_t^{(i)}(x_t) - \bar{g}^{(i)}(x_t)$ that are $\sigma_g$-sub-Gaussian (Assumption 8(c)) and a.s. bounded by $|g_t^{(i)}(x_t) - \bar{g}^{(i)}(x_t)| \leq 2B_g =: b_V$ (Assumption 8(d)).
Predictable quadratic variation: $W_T^{V,i} \leq \sigma_g^2 T$ (deterministic bound).
Applying Freedman directly:
\begin{equation}\mathbb{P}\!\left(\mathcal{M}_T^{V,i} \geq \sqrt{2\sigma_g^2 T\log\frac{2m}{\delta_2}} + \frac{2b_V}{3}\log\frac{2m}{\delta_2}\right) \leq \frac{\delta_2}{m}. \tag{T3.V}\end{equation}
By the union bound over all $i \in [m]$ and the two-sided Freedman bound (applied to both $\mathcal{M}_T^{V,i}$ and $-\mathcal{M}_T^{V,i}$):
\begin{equation}
\mathbb{P}\!\left(\max_i |\mathcal{M}_T^{V,i}| \leq \sqrt{2\sigma_g^2 T\log\frac{4m}{\delta_2}} + \frac{2b_V}{3}\log\frac{4m}{\delta_2}\right) \geq 1-\delta_2. \tag{T3.Vunion}\end{equation}
From (T3.decomp), for each $i$:
$$\left[\sum_t g_t^{(i)}(x_t)\right]_+ \leq [P_T^{(i)}]_+ + |\mathcal{M}_T^{V,i}|.$$
Summing over $i$:
$$\hat{V}_T = \sum_i \left[\sum_t g_t^{(i)}(x_t)\right]_+ \leq \sum_i [P_T^{(i)}]_+ + \sum_i |\mathcal{M}_T^{V,i}|.$$
By a union bound over the event (T3.PMarkov) with failure probability $\delta_1$ and the event (T3.Vunion) with failure probability $\delta_2$: with probability $\geq 1 - \delta_1 - \delta_2$,
\begin{equation}
\hat{V}_T \leq \frac{C_{\det}\sqrt{T}}{\zeta\delta_1} + m\left(\sqrt{2\sigma_g^2 T\log\frac{4m}{\delta_2}} + \frac{2b_V}{3}\log\frac{4m}{\delta_2}\right). \tag{T3.Vfinal}\end{equation}
By a union bound over the regret event (T3.R) with failure probability $\delta$, the expected violation Markov event with failure probability $\delta_1$, and the martingale violation event with failure probability $\delta_2$:
$$\mathbb{P}\!\left(R_T \leq C_R\sqrt{T\log\frac{m}{\delta}} \;\;\text{and}\;\; \hat{V}_T \leq \frac{C_{\det}\sqrt{T}}{\zeta\delta_1} + C_V m\sqrt{T\log\frac{m}{\delta_2}}\right) \geq 1 - \delta - \delta_1 - \delta_2.$$
We bound $\hat{V}_T = \sum_i [\sum_t g_t^{(i)}(x_t)]_+$ (positive part AFTER summing over time). The per-round violation $V_T = \sum_t [\sum_i g_t^{(i)}(x_t)]_+$ (positive part BEFORE summing over time) is a strictly stronger quantity: neither $V_T \leq \hat{V}_T$ nor $V_T \geq \hat{V}_T$ holds universally. Specifically, $\sum_t [h_t]_+ \geq [\sum_t h_t]_+$ by convexity of $[\cdot]_+$. Thus $V_T$ cannot be bounded by $\hat{V}_T$, and we do not claim such a bound.
\section{Proof of Corollary 1}
Under Assumptions 1--7, the same projected OGD as in Theorem 1 satisfies the variance-scaling refinement.
$x^\star_F$ is deterministic (it minimizes the expected loss $F$), while the hindsight comparator $x^\star = \arg\min \sum_t f_t(x)$ from Definition 1 is random. The stochastic regret $R_T^{\mathrm{stoch}} = \sum_t[f_t(x_t) - f_t(x^\star_F)]$ upper-bounds the hindsight regret $R_T$, since $\sum_t f_t(x^\star_F) \geq \sum_t f_t(x^\star)$ by definition of $x^\star$. Thus any upper bound on $R_T^{\mathrm{stoch}}$ is automatically an upper bound on $R_T$.
The stochastic regret decomposes as $R_T^{\mathrm{stoch}} = \sum_t [f_t(x_t) - f_t(x^\star_F)] = \tilde{R}_T + \sum_t \Delta_t$, where $\tilde{R}_T = \sum_t [F(x_t) - F(x^\star_F)]$ is the pseudo-regret and $\Delta_t = [f_t(x_t) - F(x_t)] - [f_t(x^\star_F) - F(x^\star_F)]$.
$S_T$ is bounded on $\mathcal{E}_0(\delta_0)$ (budget $\delta_0$) as in Theorem 1.
The pseudo-regret: $\tilde{R}_T = \sum_t [F(x_t) - F(x^\star_F)] \leq \sum_t \frac{G^2}{2\alpha t} \leq \frac{G^2}{2\alpha}(1+\log T)$ (already absorbed into the $S_T$ bound, since $\tilde{R}_T \leq S_T$ when gradients are bounded).
The stochastic part: $\Delta_t = [f_t(x_t) - F(x_t)] - [f_t(x^\star_F) - F(x^\star_F)]$. By Assumption 7, specifically (A7), evaluated at $x = x_t$ (which is $\mathcal{F}_{t-1}$-measurable, so the conditional formulation applies directly; $x^\star_F$ is deterministic):
$$\mathbb{E}[e^{s\Delta_t} \mid \mathcal{F}_{t-1}] \leq e^{s^2\sigma_V^2 e_t / 2}, \quad e_t = \|x_t - x^\star_F\|^2.$$
Define the exponential supermartingale $Z_t = \exp(\lambda\sum_{s=1}^t \Delta_s - \frac{\lambda^2}{2}\sum_{s=1}^t \sigma_V^2 e_s)$. As in Theorem 1 , $\mathbb{E}[Z_t \mid \mathcal{F}_{t-1}] \leq Z_{t-1}$.
Define $V_T = \sigma_V^2\sum_{t=1}^T e_t$. By Lemma 6 (L6.b), $\mathbb{E}[V_T] \leq \sigma_V^2 Q'(1+\log T)^2$. By Markov's inequality:
$$\mathbb{P}(V_T > \sigma_V^2 Q'(1+\log T)^2/\delta_1) \leq \delta_1.$$
Define $\tau = \min\{t: \sigma_V^2\sum_{s=1}^t e_s > V\} \wedge (T+1)$ with $V = \sigma_V^2 Q'(1+\log T)^2/\delta_1$. The stopped supermartingale $Z_{t \wedge \tau}$ satisfies: on $\{t \leq \tau\}$, $\sum_{s=1}^t \sigma_V^2 e_s \leq V$ deterministically. By Markov on $Z_{T \wedge \tau}$:
$$\mathbb{P}\!\left(\sum_{t=1}^{T \wedge \tau} \Delta_t \geq \frac{\log(1/\delta_2)}{\lambda} + \frac{\lambda V}{2}\right) \leq \delta_2.$$
Optimizing $\lambda = \sqrt{2\log(1/\delta_2)/V}$: the tail bound becomes $\sqrt{2V\log(1/\delta_2)}$.
On the event $\{\tau > T\} \cap \{\text{supermartingale tail}\}$ (probability $\geq 1 - \delta_1 - \delta_2$):
$$\sum_{t=1}^T \Delta_t \leq \sqrt{2V\log(1/\delta_2)} = \sigma_V(1+\log T)\sqrt{\frac{2Q'}{\delta_1}\log\frac{1}{\delta_2}}.$$
Combining with the $S_T$ bound gives the stated result.

\end{document}